\documentclass[11pt]{article}

\usepackage[margin=1in]{geometry}
\usepackage[numbers]{natbib}
\usepackage{color}

\usepackage{times}
\usepackage{graphicx} 

\usepackage{natbib}

\usepackage{algorithm}
\usepackage{algorithmic}

\usepackage{hyperref}

\usepackage{url}            
\usepackage[labelformat=simple]{subcaption}

\usepackage{amsmath, amsthm, amssymb, xspace, verbatim, bm}
\usepackage{color}
\usepackage{pgfplots}
\usetikzlibrary{patterns}
\usepackage{tikz}
\definecolor{bblue}{HTML}{0064FF}
\definecolor{rred}{HTML}{C0504D}
\definecolor{ggreen}{HTML}{9BBB59}
\definecolor{ppurple}{HTML}{9F4C7C}

\definecolor{geo}{rgb}{0.0, 0.0, 1.0} 
\definecolor{bsnmf}{rgb}{1.0, 0.0, 0.0} 
\definecolor{svi}{rgb}{0.0, 0.5, 0.0} 
\definecolor{saac}{rgb}{0.0, 0.0, 0.5} 
\definecolor{occam}{rgb}{0.55, 0.27, 0.07} 

\renewcommand{\algorithmicrequire}{\textbf{Input:}}
\renewcommand{\algorithmicensure}{\textbf{Output:}}

\DeclareMathAlphabet\mathbfcal{OMS}{cmsy}{b}{n}

\theoremstyle{definition}

\newtheorem{thm}{Theorem}[section]
\newtheorem{lem}[thm]{Lemma}

\newtheorem{thmDef}{Theorem}[section]
\newtheorem{defn}[thmDef]{Definition}

\newtheorem{thmrem}{Theorem}[section]
\newtheorem{rem}[thmrem]{Remark}

\newtheorem{fact}[thmrem]{Fact}
\usepackage{enumitem}

\newcommand{\bb}[1]{\left(#1\right)}
\newcommand{\ddd}[1]{\left\{#1\right\}}
\newcommand{\bas}[1]{\begin{align*}#1\end{align*}}
\newcommand{\ba}[1]{\begin{align}#1\end{align}}

\newcommand{\pushright}[1]{\ifmeasuring@#1\else\omit\hfill$\displaystyle#1$\fi\ignorespaces}

\newcommand{\txtred}[1]{}

\newcommand{\bone}{{\bf 1}}
\newcommand{\bA}{{\bf A}}
\newcommand{\bB}{{\bf B}}

\newcommand{\bC}{{\bf C}}

\newcommand{\bD}{{\bf D}}
\newcommand{\be}{{\bf e}}
\newcommand{\bE}{{\bf E}}

\newcommand{\bH}{{\bf H}}
\newcommand{\bI}{{\bf I}}

\newcommand{\bM}{{\bf M}}

\newcommand{\bO}{{\bf O}}
\newcommand{\bP}{{\bf P}}
\newcommand{\bQ}{{\bf Q}}
\newcommand{\br}{{\bf r}}
\newcommand{\bR}{{\bf R}}
\newcommand{\bT}{{\bf T}}

\newcommand{\bv}{{\bf v}}
\newcommand{\bV}{{\bf V}}
\newcommand{\bW}{{\bf W}}

\newcommand{\bx}{{\bf x}}
\newcommand{\bX}{{\bf X}}
\newcommand{\by}{{\bf y}}
\newcommand{\bY}{{\bf Y}}

\newcommand{\bZ}{{\bf Z}}
\newcommand{\bpi}{{\bf \Pi}}

\newcommand{\balpha}{{\boldsymbol \alpha}}
\newcommand{\bbeta}{{\boldsymbol \beta}}
\newcommand{\btheta}{{\boldsymbol \theta}}
\newcommand{\bTheta}{{\boldsymbol \Theta}}

\newcommand{\bbb}[1]{\left(#1\right)}
\newcommand{\ccc}[1]{\left[#1\right]}
\newcommand{\eee}[1]{\left\|#1\right\|}

\newcommand{\A}{\mathcal{A}}
\newcommand{\D}{{\mathbfcal{D}}}

\newcommand{\cS}{\mathcal{S}}
\newcommand{\I}{\mathcal{I}}

\newcommand{\R}{\mathbb{R}}

\newcommand{\cone}{\mathrm{cone}}
\newcommand{\rank}{\mathrm{rank}}

\newcommand{\cov}{\mathrm{Cov}}

\newcommand{\Var}{\mathrm{Var}}

\newcommand{\uE}{\mathrm{E}}

\newcommand{\uP}{\mathrm{P}}
\newcommand{\dir}{\mathrm{Dirichlet}}
\newcommand{\ber}{\mathrm{Bernoulli}}

\newcommand{\hx}{\hat{\bX}\xspace}
\newcommand{\hy}{\hat{\boldsymbol Y}\xspace}
\newcommand{\diag}{\text{diag}}
\newcommand{\hxp}{\hx_p}
\newcommand{\hyp}{\hy_p\xspace}
\newcommand{\cSp}{\cS_p\xspace}
\newcommand{\bxp}{\bX_p\xspace}
\newcommand{\dxp}{\bD_{21p}\xspace}
\newcommand{\dpp}{\D_{21p}\xspace}

\newcommand{\bk}{\color{black}}

\newcommand{\byp}{\bY_p\xspace}

\newcommand{\OurAlgo}{{\textsf GeoNMF}\xspace}
\newcommand{\clusteringAlgo}{PartitionPureNodes\xspace}

\newcommand{\tablefontsize}{\footnotesize}

\newcommand{\rowwise}{O_P\bbb{{\frac{{ K^2\sqrt{\log n}}}{\beta_{\mathrm{min}}^{5/2}\rho \sqrt{n}}}}}

\newcommand{\successProb}{$1- O(K^2/n^2)$}
\newcommand{\tauvalue}{\sqrt{\frac{K}{4n}\frac{\min_{i\in \mathcal{F}}\bD_2(i,i)}{\max_{i\in \mathcal{F}}\bD_2(i,i)}}}
\newcommand{\xpinverse}{O_P\bbb{{\frac{{ K^{5/2}\sqrt{\log n}}}{\beta_{\mathrm{min}}^{5/2}\rho \sqrt{n}}}}}
\newcommand{\thetaerror}{O_P\bbb{{\frac{K^3\sqrt{\log n}}{\beta_{\mathrm{min}}^{3} \rho\sqrt{n}}}}}
\newcommand{\xabsoluterow}{O_P\bbb{\frac{\sqrt{K^5\log n}}{\beta_{\mathrm{min}}^{5/2}\rho n}}}
\newcommand{\betaabsoluteerror}{O_P\bbb{\frac{{K^{5/2}\log n}}{\beta_{\mathrm{min}}^{5/2}\sqrt{\rho n}}}}
\newcommand{\betaerror}{O_P\bbb{\frac{{K^{5/2}\log n}}{\beta_{\mathrm{min}}^{5/2}\rho\sqrt{n }}}}

\begin{document}
	
	\title{On Mixed Memberships and Symmetric Nonnegative Matrix Factorizations 
	}
	\author{Xueyu Mao\thanks{Department of Computer Science. Email: \href{mailto:xmao@cs.utexas.edu}{xmao@cs.utexas.edu}}, \  Purnamrita Sarkar\thanks{Department of Statistics and Data Sciences. Email: \href{mailto:purna.sarkar@austin.utexas.edu}{purna.sarkar@austin.utexas.edu}}, \ Deepayan Chakrabarti\thanks{Department of Information, Risk, and Operations Management. Email: \href{mailto:deepay@utexas.edu}{deepay@utexas.edu}}\\ The University of Texas at Austin}
	\date{}
	
	\maketitle





%
%
%
%
%
%

	\begin{abstract}
		The problem of finding overlapping communities in networks has gained
		much attention recently.
		Optimization-based approaches use 
		non-negative matrix factorization
		(NMF) or variants, but the global optimum cannot be provably attained in general.
		Model-based approaches,
		such as the popular mixed-membership stochastic blockmodel or
		MMSB~\cite{airoldi2008mixed},
		use parameters for each node to specify the overlapping communities, but standard inference techniques cannot guarantee consistency.
		We link the two approaches, by (a) establishing
		sufficient conditions for the symmetric NMF optimization to have a unique solution under MMSB, and (b) proposing a computationally efficient algorithm called \OurAlgo that is provably optimal and hence consistent for a broad parameter regime. We demonstrate its accuracy on both simulated and real-world datasets.
	\end{abstract}



	\section{Introduction}
\label{sec:intro}
Community detection is a fundamental problem in network analysis. It has been widely used in a diverse set of applications ranging from link prediction in social networks~\cite{soundarajan2012using}, predicting protein-protein or protein-DNA interactions in biological networks~\cite{chen2006detecting}, to network protocol design such as data forwarding in delay tolerant networks~\cite{lu2015algorithms}. 

Traditional community detection assumes that every node in the network belongs to exactly one community, but many practical settings call for greater flexibility. For instance, individuals in a social network may have multiple interests, and hence are best described as members of multiple interest-based communities.
%
%
We focus on the popular mixed membership stochastic blockmodel (MMSB)~\cite{airoldi2008mixed}  where each node $i$, $i\in [n]$ has a discrete probability distribution $\btheta_i=\left(\theta_{i1}, \ldots, \theta_{iK}\right)$ over $K$ communities. The probability of linkage between nodes $i$ and $j$ depends on the degree of overlap between their communities:
\begin{align*}
\btheta_i  &\sim \dir(\balpha) & i\in[n]\\
\bP &= \rho \bTheta \bB \bTheta^T &\\
\bA_{ij} = \bA_{ji} &= \ber(\bP_{ij}) & i,j\in[n]
\end{align*}
where $\btheta_i$ is the $i$-th row of $\bTheta$, $\bA$ represents the adjacency matrix of the generated graph, and $\bB \in \mathbb{R}^{K\times K}$ is the community-community interaction matrix.
The parameter $\rho$ controls the sparsity of the graph, so WLOG, the largest entry of $\bB$ can be set to 1. The parameter $\alpha_0=\sum_i\alpha_i$ controls the amount of overlap. In particular, when $\alpha_0\rightarrow 0$, MMSB reduces to the well known stochastic blockmodel, where every node belongs to exactly one community. Larger $\alpha_0$ leads to more overlap. Since we only observe $\bA$, a natural question is: how can $\{\btheta_i\}$ and $\bB$ be recovered from $\bA$ in a way that is provably consistent? 


\subsection{\bf Prior work}
We categorize existing approaches broadly into three groups: model-based parameter inference methods, specialized algorithms that offer provable guarantees, and optimization-based methods using non-negative matrix factorization.

\noindent
{\bf Model-based methods:}
These apply standard techniques for inference of hidden variables to the MMSB model. Examples include MCMC techniques~\cite{changLDA} and variational methods~\cite{gopalan2013efficient}. While these often work well in practice, there are no proofs of consistency for these methods. The MCMC methods are difficult to scale to large graphs, so we compare against the faster variational inference methods in our experiments.

\noindent
{\bf Algorithms with provable guarantees:}
There has been work on provably consistent estimation on models similar to MMSB.~\citet{zhang2014detecting} propose a spectral method (OCCAM) for a model where the $\btheta_i$ has unit $\ell_2$ norm (unlike MMSB, where they have unit $\ell_1$ norm). In addition to the standard assumptions regarding the existence of ``pure'' nodes\footnote{This is a common assumption even for NMF methods for topic modeling, where each topic is assumed to have an anchor word (words belonging to only one topic). 
\citet{huang2016anchor}  
introduced a special optimization criterion to relax the presence of anchor words, but the optimization criterion is non-convex.}
(which only belong to a single community) and a positive-definite $\bB$, they also require $\bB$ to have equal diagonal entries, and assume that the ground truth communities has a unique optimum of a special loss function, and there is curvature around the optimum. Such assumptions may be hard to verify.
\citet{ray2014overlap} and~\citet{kaufmann2016spectral} consider models with binary community memberships. \citet{kaufmann2016spectral} show that the global optimum of a special loss function is consistent. However, achieving the global optimum is computationally intractable, and the scalable algorithm proposed by them (SAAC) is not provably consistent. 
\citet{anandkumar2014tensor} propose a tensor based approach for MMSB.
Despite their elegant solution the computational complexity is $O(n^2 K)$, which can be prohibitive for large graphs. 

\noindent
{\bf Optimization-based methods:}
If $\bB$ is positive-definite, the MMSB probability matrix $\bP$ can be written as \mbox{$\bP=\bW\bW^T$}, where the $\bW$ matrix has only non-negative entries. 
In other words, $\bW$ is the solution to a Symmetric Non-negative Matrix Factorization (SNMF) problem: $\bW = \arg\min_{{\bf X}\geq{\bf 0}}\text{loss}(\bP, {\bf X}{\bf X}^T)$ for some loss function that measure the ``difference'' between $\bP$ and its factorization. 
SNMF has been widely studied and successfully used for community detection~\cite{kuang2015symnmf,wang2011community, wang2016supervised, BNMF2011}, but typically lacks the guarantees we desire. Our paper attempts to address these issues. 

We note that~\citet{arora2012learning,arora2013practical} used NMF
to consistently estimate parameters of a topic model. However, their results cannot be easily applied to the MMSB inference problem. 
In particular, for topic models, the columns of the word-by-topic matrix specifying the probability distribution of words in a topic sum to 1. For MMSB, the rows of the node membership matrix 
sum to 1.  The relationship of this work to the MMSB problem is unclear.

\subsection{Problem Statement and Contributions}
We seek to answer two problems.

{\bf Problem 1:} {\sl Given $\bP$, when does the solution to the SNMF optimization yield the correct $\bW$?}

The difficulty stems from the fact that (a) the MMSB model may not always be identifiable, and (b) even if it is, the corresponding SNMF problem may not have a unique solution (even after allowing for permutation of communities). 

Even when the conditions for Problem 1 are met, we may be unable to find a good solution in practice. This is due to two reasons. First, we only know the adjacency matrix $\bA$, and not the probability matrix $\bP$. Second, the general SNMF problem is non-convex, and SNMF algorithms can get stuck at local optima. Hence, it is unclear if an {\em algorithm} can consistently recover the MMSB parameters. This leads to our next question.

{\bf Problem 2:} {\sl Given $\bA$ generated from a MMSB model, can we develop a fast and provably consistent algorithm to infer the parameters?}

{ Our goal is to develop a fast algorithm that provably solves SNMF for an identifiable MMSB model. Note that  generic SNMF algorithms typically do not have any provable guarantees.}



Our contributions are as follows.

\smallskip\noindent
{\bf Identifiability:} We show conditions that are sufficient for MMSB to be identifiable; specifically, there must be at least one ``pure'' exemplar of each of the $K$ clusters (i.e., a node that belongs to that community with probability $1$), and $\bB$ must be full rank.

\smallskip\noindent
{\bf Uniqueness under SNMF:} We provide sufficient conditions
under which an identifiable MMSB model is the unique solution for the SNMF problem; specifically, the MMSB probability matrix $\bP$ has a unique SNMF solution if $\bB$ is diagonal. It is important to note that MMSB with a diagonal $\bB$ still allows for interactions between different communities via members who belong to both.

\smallskip\noindent
{\bf Recovery algorithm:} We present a new algorithm, called \OurAlgo, for recovering the parameters $\{\btheta_i\}$ and $\bB$ given only the observed adjacency matrix $\bA$. The only compute-intensive part of the algorithm is the calculation of the top-$K$ eigenvalues and
eigenvectors of $\bA$, for which highly optimized algorithms
exist~\cite{press92numerical}.

\smallskip\noindent
{\bf Provable guarantees:} Under the common assumption that
$\btheta_i$ are generated from a Dirichlet($\balpha$) prior,
we prove the consistency of \OurAlgo
when $\bB$ is diagonal and there are ``pure'' nodes for
each cluster (exactly the conditions needed for uniqueness of
SNMF). We allow the sparsity parameter $\rho$ to decay with the graph size $n$. All proofs are deferred to the 
appendix.

\smallskip\noindent
{\bf Empirical validation:} On simulated networks, we compare \OurAlgo against variational methods (SVI)~\cite{gopalan2013efficient}. Since OCCAM, SAAC, and BSNMF (a Bayesian variant of SNMF~\cite{BNMF2011}) are formed under different model assumptions, we exclude these for the simulation experiments for fairness. We also run experiments on Facebook and Google Plus ego networks collected by~\citet{mcauley2014discovering};  co-authorship datasets constructed by us from DBLP~\cite{ley2002dblp} and the Microsoft academic graph (MAG)~\cite{sinha2015overview}. These networks can have up to 150,000 nodes. On these graphs we compare \OurAlgo against SVI, SAAC, OCCAM and BSNMF. We see that \OurAlgo is consistently among the top, while also being one of the fastest. This establishes that \OurAlgo achieves excellent accuracy and is computationally efficient in addition to being provably consistent.

\section{Identifiability and Uniqueness}
In order to present our results, we will now introduce some key definitions. Similar definitions appear in~\cite{zhang2014detecting}.
\begin{defn}
	A node $i \in [n]$ is called a ``pure'' node if $\exists$ $j \in [K]$ such that $\theta_{ij}=1$ and $\theta_{i\ell}=0$ for all $\ell \in [K]$, \mbox{$\ell\neq j$}.
\end{defn}

\noindent
{\bf Identifiability of MMSB.}
MMSB is not identifiable in general. Consider the following counter example.
\begin{align*}
\bM_1 &= 
  \begin{bmatrix}
  0.5 & 0.5 & 0\\
  0   & 0.5 & 0.5\\
  0.5 & 0   & 0.5
  \end{bmatrix} &
\bM_2 &= 
  \begin{bmatrix}
  0.5  & 0.25 & 0.25\\
  0.25 & 0.5  & 0.25\\
  0.25 & 0.25 & 0.5
  \end{bmatrix}
\end{align*}
It can be easily checked that the probability matrices $\bP$ generated by  the parameter set  $(\bTheta^{(1)}, \bB^{(1)}, \rho^{(1)})=(\bM_1, \bI_{3\times 3},1)$ is exactly the same as that generated by $(\bTheta^{(2)}, \bB^{(2)}, \rho^{(2)})=(\bI_{3\times 3}, 2\bM_2, 0.5)$, where $\bI_{3\times 3}$ is the identity matrix. This example can be extended to arbitrarily large $n$: for every new row $\btheta_i^{(2)}$ added to $\bTheta^{(2)}$, add the row $\btheta_i^{(1)}=\btheta_i^{(2)} \bM_1$ to $\bTheta^{(1)}$. The new rows are still non-negative and sum to $1$; it can be verified that $\bP^{(1)}=\bP^{(2)}$ even after these new node additions.

Thus, while MMSB is not identifiable in general, we can prove
identifiability under certain conditions. 
\begin{thm}[Sufficient conditions for MMSB identifiability] \label{mmsb_iden}
  Suppose parameters $\bTheta, \bB$ of the MMSB model satisfy the following conditions: (a) there is at least one pure node for each community, and (b) $\bB$ has full rank. 
  Then, MMSB is identifiable up to a permutation.
\end{thm}

Since identifiability is a necessary condition for consistent
recovery of parameters, we will assume these conditions from now on.



\noindent
{\bf Uniqueness of SNMF for MMSB model.}
Even when the MMSB model is identifiable, the SNMF optimization may not have a unique solution. In other words, given an MMSB probability matrix $\bP$, there might be multiple matrices ${\bf X}$ such that $\bP = {\bf X}{\bf X}^T$, even if $\bP$ corresponds to a unique parameter setting $(\bTheta, \bB,\rho)$ under MMSB.
For SNMF to work, $\bW=\sqrt{\rho}\bTheta \bB^{1/2}$ must the the unique SNMF solution. When does this happen?

In general, SNMF is not unique because $\bW$ can be permuted, so we consider the following definition of uniqueness.
\begin{defn}{(Uniqueness of SNMF \cite{huang2014non})} The Symmetric NMF of $\bP=\bW\bW^T$ is said to be (essentially) unique if $\bP=\tilde{\bW}\tilde{\bW}^T$ implies $\tilde{\bW}=\bW \bZ$, where $\bZ$ is a permutation matrix.
\end{defn}



\begin{thm}[Uniqueness of SNMF for MMSB] \label{unique}
  Consider an identifiable MMSB model where $\bB$ is diagonal.
  Then, its Symmetric NMF $\bW$ is unique and equals 
$\sqrt{\rho} \bTheta \bB^{1/2}$ .
\end{thm}
The above results establish that if we find a $\bW$ that is the symmetric NMF solution of $\bP$ then it is at least unique. However, two practical questions are still unanswered. First, given the non-convex nature of SNMF, how can we guarantee that we find the correct $\bW$ given $\bP$? Second, in practice we are given not $\bP$ but the noisy adjacency matrix $\bA$. Typical algorithms for SNMF do not provide guarantees even for the first question. 

\section{Provably consistent inference for MMSB}

To achieve consistent inference, we turn to the specific structure of the MMSB model. We motivate our approach in three stages. First, note that under the conditions of Theorem~\ref{unique}, the rows of $\bW$ form a {\em simplex} whose corners are formed by the pure nodes for each cluster. In addition, these corners are aligned along different axes, and hence are orthogonal to each other. Thus, if we can detect the corners of the simplex, we can recover the MMSB parameters. So the goal is to find the pure nodes from different clusters, since they define the corners.

While our goal is to get $\bW$, note that it is easy to compute $\bV\bE^{1/2}$ where $\bV,\bE$ are the eigenvectors and eigenvalues of $\bP$, i.e., $\bP=\bV\bE\bV^T$. 
Thus, $\bW \bW^T = (\bV\bE^{1/2}) (\bV\bE^{1/2})^T$. This implies that $\bW = \bV\bE^{1/2}\bQ$ for some orthogonal matrix $\bQ$ (Lemma~A.1 of~\cite{tang2013universally}). Essentially we should be able to identify the pure nodes by finding the corners of the simplex based on $\bV$ and $\bE$.

Once we have found the pure nodes, it is easy to find the rotation matrix $\bQ$ modulo a permutaion of classes, because we know that the pure nodes are on the axis for the simplex of $\bTheta \bB^{1/2}$.

Now, we note something rather striking. Let $\mathbfcal{D}$ denote the diagonal matrix with expected degrees on the diagonal. Consider the population Laplacian $\mathbfcal{D}^{-1/2}\bP\mathbfcal{D}^{-1/2}$. Its square root is given by $\mathbfcal{D}^{-1/2}\bV\bE^{1/2}$, which
has the following interesting property for equal Dirichlet parameters $\alpha_a=\alpha_0/K$. We show in Lemma \ref{lem:maxnorm} that while the resulting rows no longer fall on a simplex, the rows with the largest norm are precisely the pure nodes, for whom the norm concentrates around $\sqrt{K/n}$. \bk Thus, picking the rows with the largest norm of the square root gives us the pure nodes. From this, $\bQ$,  $\btheta_i$ for other rows and the parameters $\rho$ and $\bB$ can again be easily extracted.

Needless to say, this only answers the question for the expectation matrix $\bP$. In reality, we have a noisy adjacency matrix. Let $\hat{\bV}$ and $\hat{\bE}$ denote the matrices of eigenvectors and eigenvalues of $\bA$. We also establish in this paper that the rows of $\hat{\bV}\hat{\bE}^{1/2}$ concentrate around its population counterpart (corresponding row of $\bV\bE^{1/2}\bm{O}$ for some rotation matrix $\bm{O}$). While there are eigenvector deviation results in random matrix theory, e.g. the Davis-Kahan Theorem \cite{davis1970rotation}, these typically provide deviation results for the whole $\hat{\bV}$ matrix, not its rows. %
In a nutshell, this crucial result lets us carefully bound the errors of each step of the same basic idea executed on $\bA$, the noisy proxy for $\bP$.

\begin{algorithm}[!htbp] 
	\caption{\OurAlgo}
	\label{nmf-mmsb-pure}
	\begin{algorithmic}[1]
		\REQUIRE Adjacency matrix $\bA$; number of communities  $K$; a constant $\epsilon_0$
		\ENSURE Estimated node-community distribution matrix $\hat{\bTheta}$, Community-community interaction matrix $\hat{\bB}$, sparsity-control parameter $\hat{\rho}$;
		\STATE Randomly split the set of nodes $[n]$ into two equal-sized
		parts $\cS$ and $\bar{\cS}$.
		\STATE Obtain the top $K$ eigen-decomposition of
		$\bA(\cS, \cS)$ as $\hat{\bV}_1\hat{\bE}_1\hat{\bV}_1^T$ and of
		$\bA(\bar{\cS}, \bar{\cS})$ as $\hat{\bV}_2\hat{\bE}_2\hat{\bV}_2^T$.
		\STATE Calculate degree matrices $\bD_{2}$, $\bD_{12}$ and
		$\bD_{21}$ for the rows of $\bA(\bar{\cS},\bar{\cS})$, $\bA(\cS, \bar{\cS})$ and 
		$\bA(\bar{\cS}, \cS)$ respectively.
		\STATE $\hx=\bD_{21}^{-1/2}\bA_{21}\hat{\bV}_1\hat{\bE}_1^{-1/2}$, where $\bA_{21}=\bA(\bar{\cS},\cS).$
		\STATE  $\mathcal{F} = \left\{i: \|\hx(i,:)\|_2 \geq 
		\bbb{ 1 - 
			\epsilon_0
		}
		\max_{j} \|\hx(j,:)\|_2 
		\right\}$ 
		\normalsize
		\STATE $\cSp=\text{PartitionPureNodes}${$\left(\hx(\mathcal{F},:), \tauvalue \right)$}
		\STATE $\hxp = \hx(\cSp,:)$
		\STATE Get $\hat{\bbeta}$, where 
		$\hat{\beta}_{i} = \left\|\be_i^T\bD_{21}^{1/2}(\cSp,\cSp)\hxp
		\right\|_2^2$, $i \in [K]$
		\STATE $\hat{\bB} = \mathrm{diag}(\hat{\bbeta})$
		\STATE $\hat{\rho} = \max_{i} \hat{\bB}_{ii}$
		\STATE $\hat{\bB} = \hat{\bB} / \hat{\rho}$
		\STATE $\hat{\bTheta}(\bar{\cS},:) = \bD_{21}^{1/2} \hx \hxp^{-1} \bD_{21}^{-1/2}(\cSp,\cSp)$
		\STATE Repeat steps with $\bD_{12}$, $\bA_{12}$, $\hat{\bV}_2$, and
		$\hat{\bE}_2$ to obtain parameter estimates for the remaining bipartition.
	\end{algorithmic}
\end{algorithm}

Algorithm~\ref{nmf-mmsb-pure} shows
our NMF algorithm based on these geometric intuitions 
for inference under MMSB (henceforth, \OurAlgo). The complexity of \OurAlgo is dominated by the one-time eigen-decomposition in step $2$. Thus this algorithm is fast and scalable. The consistency of parameters inferred under \OurAlgo is shown in the next section.

\begin{algorithm}[!htbp] 
	\caption{PartitionPureNodes}
	\label{algo:clustering}
	\begin{algorithmic}[1]
		\REQUIRE  Matrix $\bM \in \R^{m\times K}$, where each row represents a pure node; a constant $\tau$
		\ENSURE  A set $S$ consisting of one pure node from each cluster.
		\STATE $S=\{\}$, $C=\{\}$.
		\WHILE{ $C\neq [m]$  }
		\STATE Randomly pick one index from $[m]\setminus C$, say $s$
		\STATE $S = S \cup \{s\}$
		\STATE $C = C \cup \{i\in[m]\setminus C:\|\bM(s,:)-\bM(i,:)\|\leq \tau\}$
		\ENDWHILE
	\end{algorithmic}
\end{algorithm}

\begin{rem}
	Note that Algorithm~\ref{nmf-mmsb-pure} produces two sets of parameters for the two partitions of the graph $\cS$ and $\bar{\cS}$. In practice one may need to have parameter estimates of the entire graph. While there are many ways of doing this, the most intuitive way would be to look at the set of pure nodes in $\cS$ (call this $\cS_p$) and those in $\bar{\cS}$  (call this $\bar{\cS}_p$). If one looks at the subgraph induced by the union of all these pure nodes, then with high probability, there should be $K$ connected components, which will allow us to match the communities.
	
	Also note that Algorithm~\ref{algo:clustering} may return $k\neq K$ clusters. However, we show in Lemma~\ref{lem:clustering} that the pure nodes extracted by our algorithm will be highly separated and with high probability we will have $k=K$ for an appropriately chosen $\tau$.
	
	Finally, we note that, in our implementation, we construct the candidate pure node set $\mathcal{F}$ (step 5 of Algorithm~\ref{nmf-mmsb-pure}) by finding all nodes with norm within $\epsilon_0$ multiplicative error of the largest norm. We increase $\epsilon_0$ from a small value, until $\hx_p$ has condition number close to one. This is helpful when $n$ is small, where asymptotic results do not hold.
\end{rem}


	\section{Analysis}
We want to prove that the sample-based estimates $\hat{\bTheta}$, $\hat{\bB}$ and $\hat{\rho}$ concentrate around the corresponding population parameters $\bTheta$, $\bB$, and $\rho$ after appropriate normalization. We will show this in several steps, which follow the steps of \OurAlgo.

For the following statements, denote $\beta_{\mathrm{min}}=\min_a \bB_{aa}$,
$\bTheta_{2}=\bTheta(\bar{\cS},:)$, where $\bar{\cS}$ is one of the random bipartitions of $[n]$.
Let
$\D_{21}$ be the population version of $\bD_{21}$ defined in Algorithm~\ref{nmf-mmsb-pure}. Also let $\hx_i=\be_i^T\bD_{21}^{-1/2}\bA_{21}\hat{\bV}_1\hat{\bE}_1^{-1/2}$ and its population version $\bX_i=\sqrt{\rho}\cdot\be_i^T\D_{21}^{-1/2}\bTheta_2 \bB^{1/2}$ for $i\in[\frac{n}{2}]$.

First we show the pure nodes have the largest row norm of the population version of $\hx$. 

\begin{lem}
	\label{lem:maxnorm}
	Recall that $\bX \in \R^{\frac{n}{2}\times K}$. If $\bTheta\sim \text{Dirichlet}(\balpha)$ \bk with $\alpha_i=\alpha_0/K$, then
	$\forall i\in[\frac{n}{2}]$, 
	\bas{
		\left\| \bX_i \right\|_2^2 
		\leq \frac{2K}{ n} \max_{a}\theta_{ia} {\bbb{1+O_P\bbb{{\sqrt{\frac{K\log n}{n}}}} }}
	}
	with probability larger than $1-O(1/n^3)$.

	In particular, if node $i$ of subgraph $\bA(\bar{\cS},\bar{\cS})$ is a pure node ($\max_a \theta_{ia}=1$),
		\bas{
			\left\| \bX_i \right\|_2^2 \in \frac{2K}{n} {\ccc{1- O_P\bbb{{\sqrt{\frac{K\log n}{n}}}},1+ O_P\bbb{{\sqrt{\frac{K\log n}{n}}}} }}.
		}
\end{lem}
\medskip
\noindent
{\bf Concentration of rows of $\hx$.}
We must show that the rows of the sample  $\hat{\bX}$ matrix
concentrate around a suitably rotated population version.
While it is known that $\hat{\bV}$ concentrates around suitably rotated $\bV$ (see the variant of Davis-Kahan Theorem presented in~\cite{yu2015useful}), these results are for {\em columns} of the $\bV$ matrix, not for each {\em row}. The trivial bound for row-wise error would be to upper bound it by the total error, which is too crude for our purposes. To get row-wise convergence, we use sample-splitting (similar ideas can be found in~\cite{mcsherry2001spectral,chaudhuri2012spectral}), as detailed in steps 1 to 4 of \OurAlgo. The key idea is to split the graph in two parts and project the adjacency matrix of one part onto eigenvectors of another part. Due to independence of these two parts, one can show concentration.  

\begin{thm} 
	\label{thm:entrywise}
	Consider an adjacency matrix $\bA$ generated from MMSB$(\bTheta,\bB,\rho$), where  $\bTheta\sim \text{Dirichlet}(\balpha)$ \bk with $\alpha_i=\alpha_0/K$, whose parameters satisfy the conditions of Theorem~\ref{unique}. 
	If $\rho n = \Omega(\log n)$,
	then $\exists$ orthogonal matrix $\bm{O} \in	\mathbb{R}^{K\times K}$ that $\forall\ i \in [\frac{n}{2}]$,
	\bas{
		\frac{
			\| \hx_i - \bX_i\bm{O} \|_2
		}{
			\left\|	\bX_i\right\|_2
		} 
		=\rowwise
	}
	with probability larger than \successProb.
\end{thm}

Thus, the sample-based quantity for {\em each row} $i$ converges to its
population variant.

\medskip
\noindent
{\bf Selection of pure nodes.}
\OurAlgo selects the nodes with (almost) the highest norm. We prove
that this only selects nearly pure nodes.
Let $\epsilon' = \rowwise$ represent the row-wise error
term from Theorem~\ref{thm:entrywise}.

\begin{lem}
\label{lem:errinPure}
Let $\mathcal{F}$ be the set of nodes with \mbox{$\|\hx_i\|_2\geq (1-\epsilon_0)\max_j \|\hx_j\|_2$}. Then $\forall i\in \mathcal{F}$, 
\bas{\max_a\theta_{ia}\geq 1-O_P(\epsilon_0+\epsilon')}
with probability larger than \successProb. 
\end{lem}

We choose $\epsilon_0=O_P(\epsilon')$ and it is straightforward to show by Lemmas~\ref{lem:maxnorm}, \ref{lem:errinPure}, and Theorem \ref{thm:entrywise} that if $\epsilon_0\geq 2\epsilon'$, then $\mathcal{F}$ includes all pure nodes from all $K$ communities.

\medskip
\noindent
{\bf Clustering of pure nodes.}
Once the (nearly) pure nodes have been selected, we run \clusteringAlgo (Algorithm~\ref{algo:clustering}) on them. 
We show that these nodes can form exactly $K$ well separated clusters and each cluster only contains nodes whose $\btheta$ are peaked on the same element, and \clusteringAlgo can select exactly one node from each of the $K$ communities.

\begin{lem}
	\label{lem:clustering}
	Let { $\tau=\tauvalue$}, where $\mathcal{F}$ is defined in step 5 of Algorithm~\ref{nmf-mmsb-pure}. If all conditions in Theorem~\ref{thm:entrywise} are satisfied, then
	\clusteringAlgo($\hx(\mathcal{F},:), \tau$) returns 
	one (nearly) pure node from each of the underlying $K$ communities with probability larger than \successProb. 
\end{lem}


\medskip
\noindent
{\bf Concentration of $(\hat{\bTheta}, \hat{\bB}, \hat{\rho})$.}
\OurAlgo recovers $\bTheta$ using $\bD$, $\hx$, and its pure
portion $\hxp$ (via the inverse $\hxp^{-1}$). We first prove that
$\hxp^{-1}$ concentrates around its expectation.

\begin{thm}
	\label{thm:hxpinv}
	Let $\cSp$ be the set of of pure nodes extracted using our algorithm.
	Let $\hxp$ denote the rows of $\hx$ indexed by $\cSp$. Then, for the orthogonal matrix $\bm{O}$ from Theorem \ref{thm:entrywise},
	\bas{
		\frac{\|\hxp^{-1}-\bbb{\bxp\bm{O}}^{-1}\|_F}{ \|\bxp^{-1}\|_F}=
		\xpinverse
	}
	with probability larger than \successProb.
\end{thm}

Next, we shall prove consistency for
$\hat{\bTheta}_2:=\hat{\bTheta}(\bar{\cS},:)$; the proof for $\hat{\bTheta}(\cS,:)$
is similar.
Let $\dxp = \bD_{21}(\cSp,\cSp)$.

\begin{thm}
	\label{thm:theta}
	Let $\hat{\bTheta}_2 = \bD_{21}^{1/2}\hx \hxp^{-1} \dxp^{-1/2}$, then $\exists$ a permutation matrix $\bpi\in\R^{K\times K}$ such that
	\bas{
		\frac{\|\hat{\bTheta}_2-\bTheta_2\bpi\|_F}{\|\bTheta_2\|_F}=
		\thetaerror
	}
	with probability larger than \successProb.
\end{thm}

Recall that $\bB$ and $\hat{\bB}$ are both diagonal matrices, with diagonal components $\{\beta_a\}$ and $\{\hat{\beta}_a\}$ respectively.
\begin{thm}
	\label{thm:beta}
	Let $\hat{\rho}\hat{\beta}_a =
	\|\be_a^T\bD_{21}^{1/2}(\cSp,\cSp)\hxp \|_2^2$.
	Then, $\exists$ a permutation matrix $\bpi\in\R^{K\times K}$ such that $\forall a\in[K]$,
		\bas{
			\hat{\rho}\hat{\beta}_a \in  \rho\beta_{a'} \left[ 1 - 
			\betaerror,1+\betaerror
			\right]	
		}
	for some $a'$ such that $\bpi_{a'a}=1$, with probability larger than \successProb.
\end{thm}

\begin{rem}
While the details of our algorithms were designed for obtaining rigorous theoretical guarantees, many of these can be relaxed in practice. For instance, while we require the Dirichlet parameters to be equal, leading to balanced cluster sizes, real data experiments show that our algorithm works well for unbalanced settings as well. Similarly, the algorithm assumes a diagonal $\bB$ (which is sufficient for uniqueness), but empirically works well even in the presence of off-diagonal noise. Finally, splitting the nodes into $\mathcal{S}$ and $\bar{\mathcal{S}}$ is not needed in practice.
	\end{rem}

\section{Experiments}
\label{sec:exp}



We present results on simulated and real-world datasets. Via simulations, we evaluate the sensitivity of \OurAlgo to the various MMSB parameters: the skewness of the diagonal elements of $\bB$ and off-diagonal noise, the Dirichlet parameter $\balpha$ that controls the degree of overlap, the sparsity parameter $\rho$, and the number of communities $K$.
Then, we evaluate \OurAlgo on Facebook and Google Plus ego networks, and co-authorship networks with upto 150,000 nodes constructed from DBLP and the Microsoft Academic Network.


\smallskip\noindent
{\bf Baseline methods:} 
For the real-world networks, we compare \OurAlgo against the following methods\footnote{We were not to run \citet{anandkumar2014tensor}'s main (GPU) implementation of their algorithm because a required library CULA is no longer open source, and a complementary CPU implementation did not yield good results with default settings.}:

\begin{itemize}
	\item Stochastic variational inference (SVI) for MMSB \cite{gopalan2013efficient},
	\item a Bayesian variant of SNMF for overlapping community detection (BSNMF)~\cite{BNMF2011},
	\item the OCCAM algorithm~\cite{zhang2014detecting} for recovering mixed memberships, and
	\item the SAAC algorithm~\cite{kaufmann2016spectral}.
\end{itemize}
For the simulation experiments, we only compare \OurAlgo against SVI, since these are the only two methods based specifically on the MMSB model.
BSNMF has a completely different underlying model, 
OCCAM requires rows of $\bTheta$ to have unit $\ell_2$ norm and $\bB$ to have equal diagonal elements, and SAAC requires $\bTheta$ to be a binary matrix, while MMSB requires rows of $\bTheta$ to have unit $\ell_1$ norm.


Since the community identities can only be recovered up-to a permutation,  in both simulated and real data experiments, we figure out the order of the communities  using the well known Munkres algorithm in \cite{munkres1957algorithms}.
\begin{figure*}[!htbp]
	\centering
	\begin{subfigure}[b]{0.33\textwidth}
		\includegraphics[width=\textwidth]{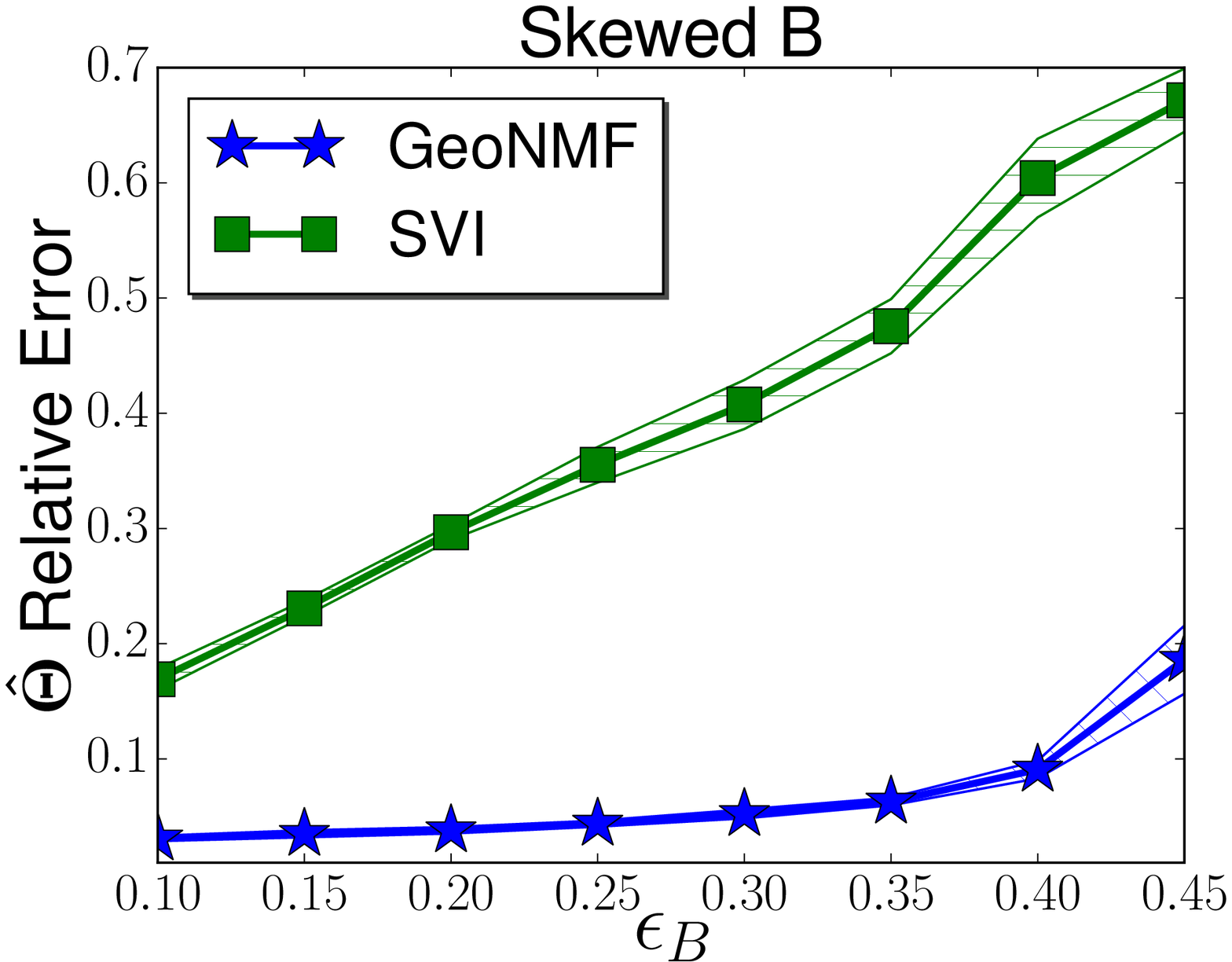}
		\caption{\tabular[t]{@{}l@{}}$\rho = 1, \bB=\text{diag}(\bbeta/\max_i\beta_i)$, \\ $\bbeta=(0.5 -\epsilon_B,0.5,0.5+\epsilon_B)$\endtabular}
		\label{skew_B_ratio_theta}
	\end{subfigure}%
	~
	\begin{subfigure}[b]{0.33\textwidth}
		\includegraphics[width=\textwidth]{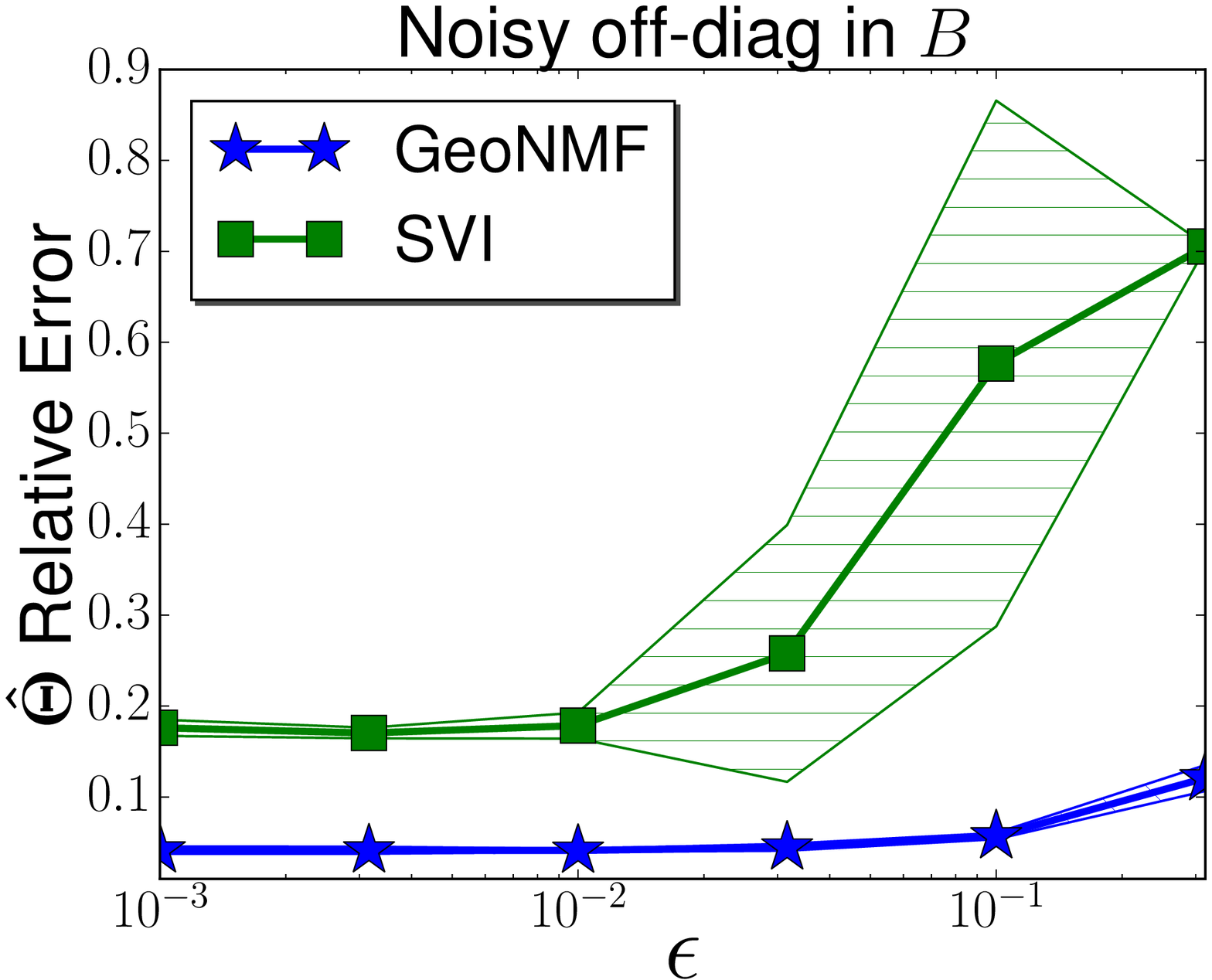}
		\caption{\tabular[t]{@{}l@{}}{$\bB=\text{diag}(\bbeta-\epsilon \cdot \bone_K)+\epsilon \cdot \bone_K\bone_K^T$}, \\ $\bbeta=(0.6,0.8,1), \rho=1$\endtabular}
		\label{noisy_B_ratio_theta}
	\end{subfigure}%
	~
	\begin{subfigure}[b]{0.33\textwidth}
		\includegraphics[width=\textwidth]{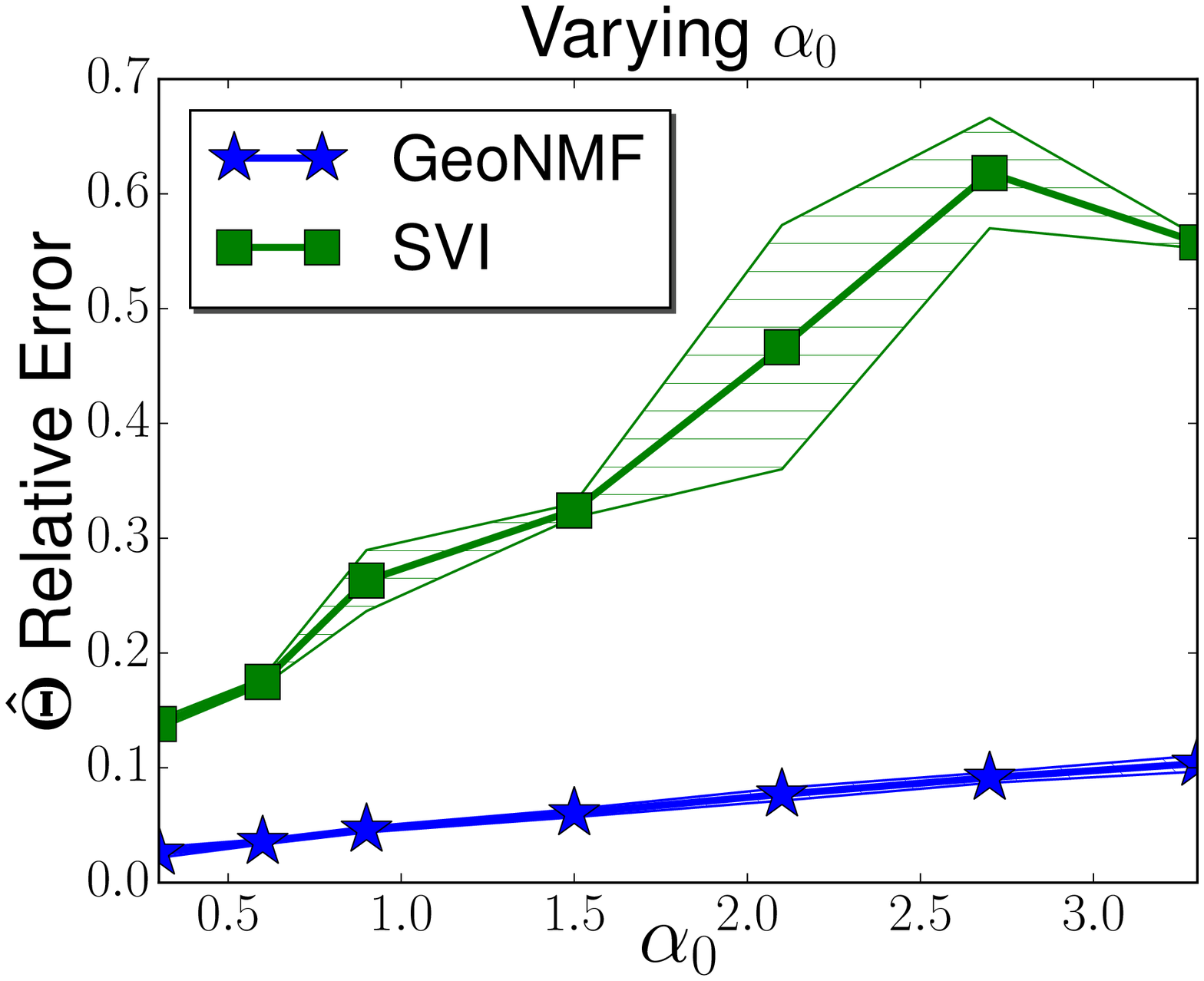}
		\caption{\tabular[t]{@{}l@{}}$\rho = 0.7, \bB=\text{diag}(\bbeta$), \\ $\bbeta=(0.4,0.7,1)$\endtabular}
		\label{changing_alpha_ratio_theta}
	\end{subfigure}%
	\\
	\begin{subfigure}[b]{0.33\textwidth}
		\includegraphics[width=\textwidth]{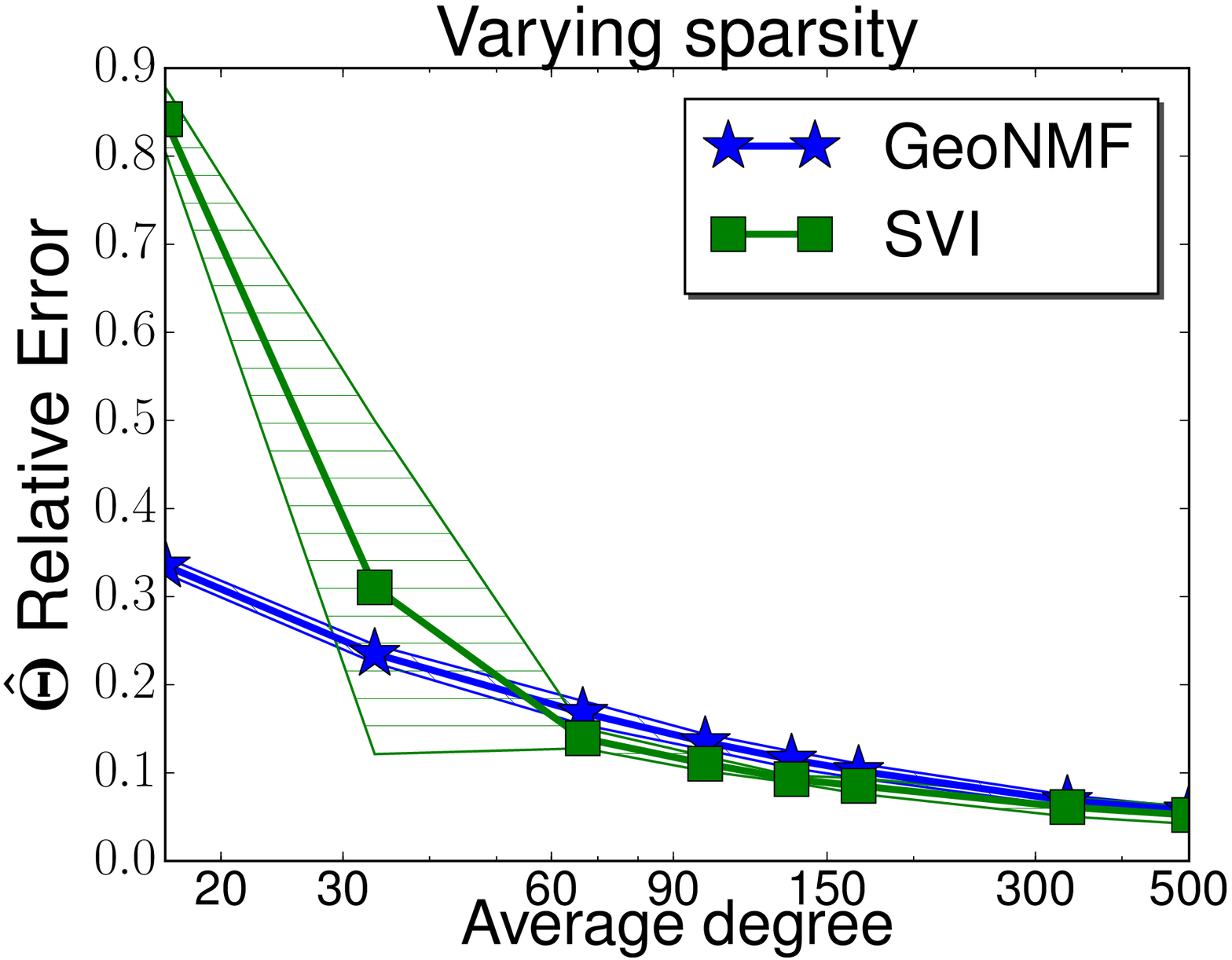}
		\caption{\tabular[t]{@{}l@{}}$\bB=\text{diag}(\bbeta)$, \\ $\bbeta=(1,1,1), \rho=1$\endtabular}
		\label{rho_sensi_ratio_theta}
	\end{subfigure}%
	~
	~
	\begin{subfigure}[b]{0.33\textwidth}
		\includegraphics[width=\textwidth]{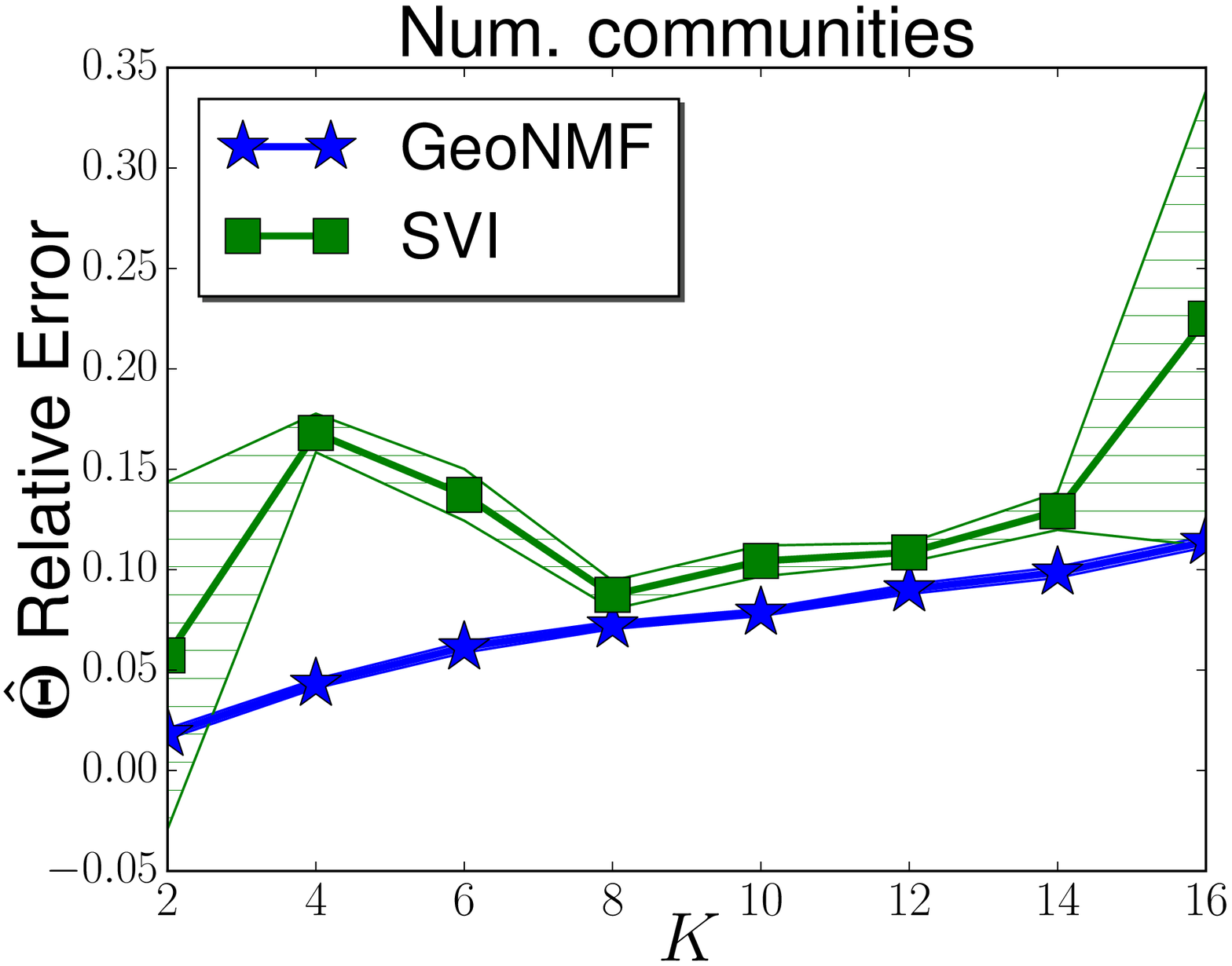}
		\caption{\tabular[t]{@{}l@{}}{$\bB=\text{diag}(0.35 \cdot \bone_K + 0.65 \cdot \br_K )$}, \\ $\br_K=\text{rand}(K,1), \rho=1$\endtabular}
		\label{K_sensi_ratio_theta}
	\end{subfigure}%
	~
	\begin{subfigure}[b]{0.33\textwidth}
		\includegraphics[width=\textwidth]{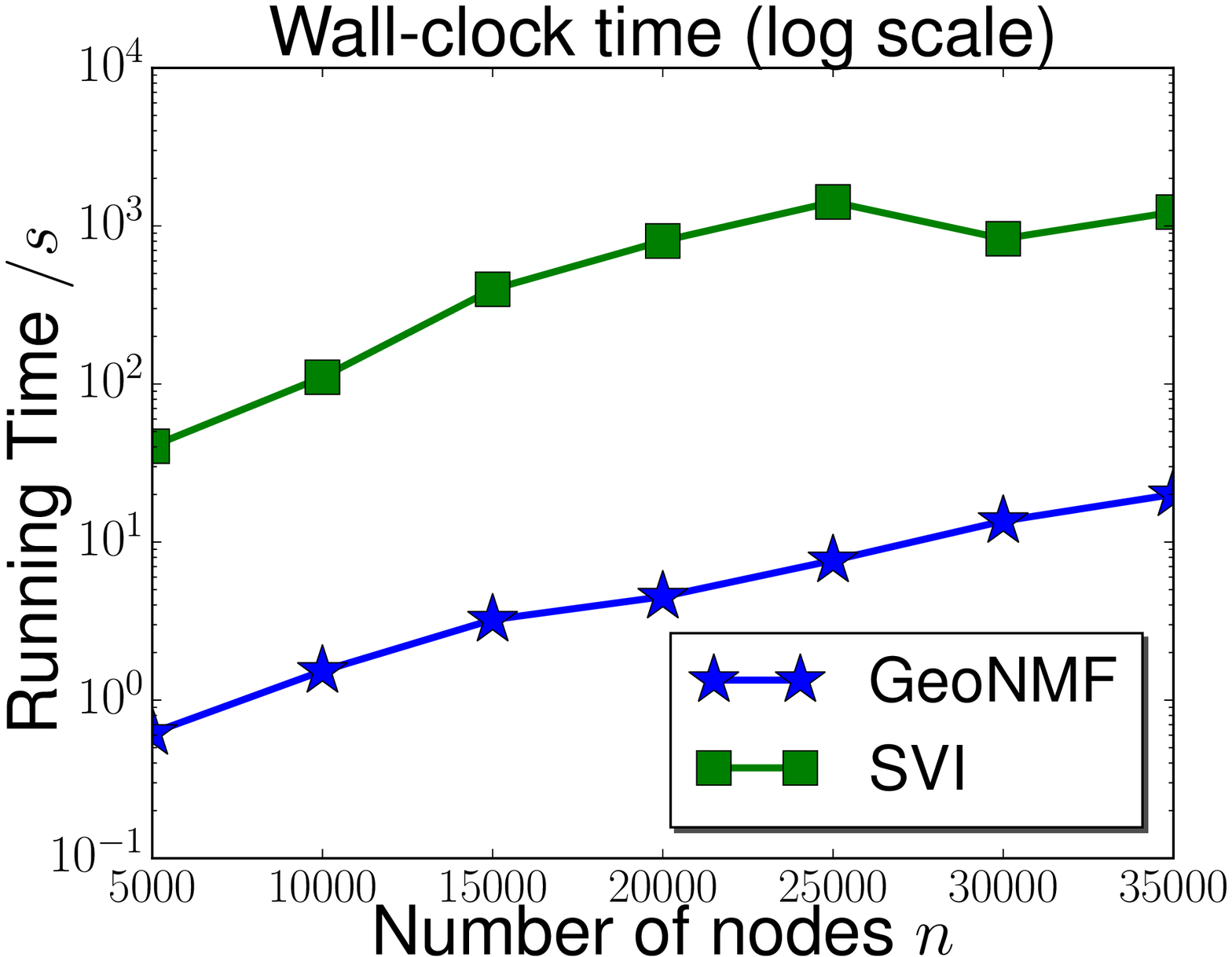}
		\caption{\tabular[t]{@{}l@{}}{$\bB=\text{diag}(0.5 \cdot \bone_K + 0.5 \cdot \br_K )$}, \\ $\br_K=\text{rand}(K,1), \rho=1$\endtabular}
		\label{wall_clock_time_N}
	\end{subfigure}%
	\caption{\subref{skew_B_ratio_theta}-\subref{K_sensi_ratio_theta} Simulation results for varying parameters. \subref{wall_clock_time_N} Running time.}\label{simu_exp}
\end{figure*}

\subsection{Simulated data}
Our simulations with the MMSB model are shown in 
\textit{Figure}~\ref{simu_exp}. We use $\alpha_i=\alpha_0/K$ for $i\in [K]$. While this leads to balanced clusters, note that the real datasets have clusters of different sizes and we will show that \OurAlgo works consistently well even for those networks (see Section~\ref{sec:realdata}). Unless otherwise stated, we set $n=5000$, $K=3$, and $\alpha_0=1$. 

\smallskip\noindent
{\bf  Evaluation Metric:}
Since we have ground truth $\bTheta$, we report the relative error  of the inferred MMSB parameters $\bTheta$ defined as  $\min\limits_{\bpi}\frac{\left\|\hat{\bTheta} - \bTheta\bpi\right\|_F}{\left\|\bTheta\right\|_F}$. Here the minimum is taken over all $K\times K$ permutation matrices. For each experiment, we report the average and the standard deviation over $10$ random samples. Since all the baseline algorithms only return $\hat{\bTheta}$, we only report relative error of that.

\smallskip\noindent
{\bf Sensitivity to skewness of the diagonal of $\bB$:}
Let $\bbeta=\diag(\bB)$. For skewed $\bbeta$, different communities have different strengths of connection. We use $\bbeta=(0.5-\epsilon_B,0.5,0.5+\epsilon_B)$ and plot the relative error against varying $\epsilon_B$.  \textit{Figure}~ \ref{skew_B_ratio_theta} shows that \OurAlgo has much smaller error than SVI, and is robust to $\bbeta$ over a wide range.
 
 \smallskip\noindent
 {\bf Sensitivity to off-diagonal element $\bB$:}
While SNMF is identifiable only for diagonal $\bB$, we still test \OurAlgo in the setting where all off-diagonal entries of $\bB$ have noise $\epsilon$. \textit{Figure}~\ref{noisy_B_ratio_theta} shows once again that \OurAlgo is robust to such noise, and is much more accurate than SVI.

\smallskip\noindent
{\bf Sensitivity to $\alpha_0$:} In \textit{Figure}~ \ref{changing_alpha_ratio_theta}, the relative error is plotted against increasing $\alpha_0$; larger values corresponding to larger overlap between communities. Accuracy degrades with increasing overlap, as expected, but \OurAlgo is much less affected than SVI.



\smallskip\noindent
{\bf Sensitivity to $\rho$:}
\textit{Figure}~\ref{rho_sensi_ratio_theta} shows relative error against increasing $\rho$. For dense networks, both \OurAlgo and SVI perform similarly, but the error of SVI increases drastically in the sparse regime (small $\rho$). 

\smallskip\noindent
{\bf Scalability:} 
\textit{Figure} \ref{wall_clock_time_N} shows the wall-clock time for networks of different sizes.  Both \OurAlgo and SVI scale linearly with the number of nodes, but SVI is about 100 times slower than \OurAlgo. 
%
%

\subsection{Real-world data}
\label{sec:realdata}

\smallskip\noindent
{\bf Datasets:} 
For real-data experiments, we use two kinds of networks:

\begin{itemize}
	\item {\em Ego networks}: We use the Facebook and Google Plus~(G-plus) ego networks, where each node can be part of multiple ``circles'' or ``communities.''
	\item {\em Co-authorship networks\footnote{Available at {\url{http://www.cs.utexas.edu/~xmao/coauthorship}}}}: We construct co-authorship networks from DBLP (each community is a group of conferences), and from the Microsoft Academic Graph (each community is denoted by a ``field of study'' (FOS) tag).
%
	Each author's $\btheta$ vector is constructed by normalizing the number of papers he/she has published in conferences in a subfield (or papers that have the FOS tag). 
\end{itemize}
\begin{table*}[!htbp]
	\centering
	\tablefontsize
	\caption{ Network statistics}
	\scalebox{0.94}{\begin{tabular}{c|c|c|c|c|c|c|c|c|c}
		\tablefontsize Dataset  & \tablefontsize Facebook                & \tablefontsize G-plus                   & \tablefontsize DBLP1    & \tablefontsize DBLP2    & \tablefontsize DBLP3    & \tablefontsize DBLP4    & \tablefontsize DBLP5 & \tablefontsize MAG1 & \tablefontsize MAG2 \\
		\hline
		\hline
		\tablefontsize $\#$ nodes $n$       & \tablefontsize 362.0 ($\pm$ 148.5) & \tablefontsize 656.2 ($\pm$ 422.0) & \tablefontsize 30,566    & \tablefontsize 16,817    & \tablefontsize 13,315    & \tablefontsize 25,481    & \tablefontsize  42,351     & \tablefontsize  142,788    & \tablefontsize    108,064  \\
		\hline
		\tablefontsize $\#$ communities $K$ & \tablefontsize 2.3 ($\pm$0.58)     & \tablefontsize 2.8 ($\pm$1.74)     & \tablefontsize 6        & \tablefontsize 3        & \tablefontsize 3        & \tablefontsize 3        & \tablefontsize   4    & \tablefontsize  3    & \tablefontsize  3    \\
		\hline
		\tablefontsize Average Degree       & \tablefontsize 56.8 ($\pm$32.3)   & \tablefontsize    103.4 ($\pm$74.9)                     & \tablefontsize 8.9 & \tablefontsize 7.6 & \tablefontsize 8.5 & \tablefontsize 5.2  & \tablefontsize   6.8    & \tablefontsize    12.4  & \tablefontsize  16.0    \\
		\hline
		Overlap $\%$     & \tablefontsize 19.3($\pm$29.8)       & \tablefontsize 26.5 ($\pm$32.4)     & \tablefontsize 18.2 & \tablefontsize 14.9 & \tablefontsize 21.1 & \tablefontsize 14.4 & \tablefontsize   18.5    & \tablefontsize   3.3   & \tablefontsize   
		3.8
	\end{tabular} }
	\label{table:net_stats}
\end{table*}
We preprocessed the networks by recursively removing isolated nodes, communities without any pure nodes, and nodes with no community assignments. For the ego networks we pick networks with at least 200 nodes and the average number of nodes per community ($n/K$) is at least 100, giving us 3 Facebook and and 40 G-plus networks. For the co-authorship networks, all communities have enough pure nodes, and after removing isolated nodes, the networks have more than 200 nodes and $n/K$ is larger than 100. The statistics of the networks (number of nodes, average degree, number of clusters, degree of overlap etc.) are shown in \textit{Table}~\ref{table:net_stats}.  The overlap ratio is the number of overlapping nodes divided by the number of nodes. The different networks have the following subfields:

\begin{itemize}
	\item DBLP1: Machine Learning, Theoretical Computer Science, Data Mining, Computer Vision, Artificial Intelligence, Natural Language Processing
	\item DBLP2: Networking and Communications, Systems, Information Theory
	\item DBLP3: Databases, Data Mining, World Web Wide
	\item DBLP4: Programming Languages, Software Engineering, Formal Methods
	\item DBLP5: Computer Architecture, Computer Hardware, Real-time and Embedded Systems, Computer-aided Design
	\item MAG1: Computational Biology and Bioinformatics, Organic Chemistry, Genetics
	\item MAG2: Machine Learning, Artificial Intelligence, Mathematical Optimization
\end{itemize}

\smallskip\noindent
{\bf  Evaluation Metric:}
For real data experiments, we construct $\bTheta$ as follows. For the ego-networks every node has a binary vector which indicates which circle (community) each node belongs to. We normalize this to construct $\bTheta$. For the DBLP and Microsoft Academic networks we construct a row of $\bTheta$ by normalizing the number of papers an author has in different conferences (ground truth communities). 
We present the averaged Spearman rank correlation coefficients (RC) between $\bTheta(:,a)$, $a\in[K]$ and $\hat{\bTheta}(:,\sigma(a))$, where $\sigma$ is a permutation of $[K]$. The formal definition is:
\bas{
	\mathrm{RC}_{\mathrm{avg}}(\hat{\bTheta},\bTheta)&=\frac{1}{K}\max_{\sigma}\sum_{i=1}^{K} \text{RC}(\hat{\bTheta}(:,i),\bTheta(:,\sigma(i))).
}

It is easy to see that 
$\mathrm{RC}_{\mathrm{avg}}(\hat{\bTheta},\bTheta)$ takes value from -1 to 1, and higher is better. Since SAAC returns binary assignment, we compute its $\text{RC}_\text{avg}$ against the binary ground truth.

\input{bar_chart}

\smallskip\noindent
{\bf Performance:} 
We report the $\mathrm{RC}_{\mathrm{avg}}$ score in \textit{Figure}~\ref{snap_rc} averaged over different Faceboook and G-plus networks; in \textit{Figure}~\ref{dblp_rc} for five DBLP networks, and in \textit{Figure}~\ref{mag_rc} for two MAG networks. We show the time in seconds (log-scale) in \textit{Figure}~\ref{snap_time} averaged over Facebook and G-plus networks; in \textit{Figure}~\ref{dblp_time} for DBLP networks and in \textit{Figure}~\ref{mag_time} for MAG networks. We averaged over the Facebook and G-plus networks because all the performances were similar.
\begin{itemize}
	\item For small networks like Facebook and G-plus, all algorithms perform equally well both in speed and accuracy, although \OurAlgo is fast even for relatively larger G-plus networks.
	\item DBLP is sparser, and as a result the overall rank correlation decreases. However, \OurAlgo consistently performs well . While for some networks, BSNMF and OCCAM  have comparable $\mathrm{RC}_{\mathrm{avg}}$, they are much slower than \OurAlgo. 
	\item MAG is larger (hundreds of thousands of nodes)  than DBLP. For these networks we could not even run BSNMF because of memory issues. Again, \OurAlgo performs consistently well while outperforming others in speed.
\end{itemize}

\smallskip\noindent
{\bf Estimating $K$:} While we assume that $K$ is known apriori, $K$ can be estimated using the USVT estimator~\cite{chatterjee2015usvt}. For the simulated graphs, when average degree is above ten, USVT estimates $K$ correctly. However for the real graphs, which are often sparse, it typically overestimates the true number of clusters.

	\section{Conclusions}
\label{sec:conc}
This paper explored the applicability of symmetric NMF algorithms for
inference of MMSB parameters. We showed broad conditions that ensure
identifiability of MMSB, and then proved sufficiency conditions for
the MMSB parameters to be uniquely determined by a general symmetric
NMF algorithm. Since general-purpose symmetric NMF algorithms do not
have optimality guarantees, we propose a new algorithm, called
\OurAlgo, that adapts symmetric NMF specifically to MMSB. \OurAlgo is
not only provably consistent, but also shows good accuracy
in simulated and real-world experiments, while also being among the fastest approaches.


	

%


%
%
%
%
%
%
\newpage
	{
		\bibliographystyle{plainnat}
		\bibliography{references}

\begin{thebibliography}{31}
\providecommand{\natexlab}[1]{#1}
\providecommand{\url}[1]{\texttt{#1}}
\expandafter\ifx\csname urlstyle\endcsname\relax
  \providecommand{\doi}[1]{doi: #1}\else
  \providecommand{\doi}{doi: \begingroup \urlstyle{rm}\Url}\fi

\bibitem[Airoldi et~al.(2008)Airoldi, Blei, Fienberg, and
  Xing]{airoldi2008mixed}
Edoardo~M Airoldi, David~M Blei, Stephen~E Fienberg, and Eric~P Xing.
\newblock Mixed membership stochastic blockmodels.
\newblock \emph{Journal of Machine Learning Research}, 9:\penalty0 1981--2014,
  2008.

\bibitem[Anandkumar et~al.(2014)Anandkumar, Ge, Hsu, and
  Kakade]{anandkumar2014tensor}
Animashree Anandkumar, Rong Ge, Daniel~J Hsu, and Sham~M Kakade.
\newblock A tensor approach to learning mixed membership community models.
\newblock \emph{Journal of Machine Learning Research}, 15\penalty0
  (1):\penalty0 2239--2312, 2014.

\bibitem[Arora et~al.(2012)Arora, Ge, and Moitra]{arora2012learning}
Sanjeev Arora, Rong Ge, and Ankur Moitra.
\newblock Learning topic models--going beyond svd.
\newblock In \emph{Foundations of Computer Science (FOCS), 2012 IEEE 53rd
  Annual Symposium on}, pages 1--10. IEEE, 2012.

\bibitem[Arora et~al.(2013)Arora, Ge, Halpern, Mimno, Moitra, Sontag, Wu, and
  Zhu]{arora2013practical}
Sanjeev Arora, Rong Ge, Yonatan Halpern, David~M Mimno, Ankur Moitra, David
  Sontag, Yichen Wu, and Michael Zhu.
\newblock A practical algorithm for topic modeling with provable guarantees.
\newblock In \emph{ICML}, pages 280--288, 2013.

\bibitem[Chang(2012)]{changLDA}
Jonathan Chang.
\newblock {LDA}: Collapsed gibbs sampling methods for topic models, 2012.
\newblock URL \url{http://cran.r-project.org/web/packages/lda/index.html}.

\bibitem[Chatterjee et~al.(2015)]{chatterjee2015usvt}
Sourav Chatterjee et~al.
\newblock Matrix estimation by universal singular value thresholding.
\newblock \emph{The Annals of Statistics}, 43\penalty0 (1):\penalty0 177--214,
  2015.

\bibitem[Chaudhuri et~al.(2012)Chaudhuri, Graham, and
  Tsiatas]{chaudhuri2012spectral}
Kamalika Chaudhuri, Fan~Chung Graham, and Alexander Tsiatas.
\newblock Spectral clustering of graphs with general degrees in the extended
  planted partition model.
\newblock In \emph{COLT}, volume~23, pages 35--1, 2012.

\bibitem[Chen and Yuan(2006)]{chen2006detecting}
Jingchun Chen and Bo~Yuan.
\newblock Detecting functional modules in the yeast protein--protein
  interaction network.
\newblock \emph{Bioinformatics}, 22\penalty0 (18):\penalty0 2283--2290, 2006.

\bibitem[Davis and Kahan(1970)]{davis1970rotation}
Chandler Davis and William~Morton Kahan.
\newblock The rotation of eigenvectors by a perturbation. iii.
\newblock \emph{SIAM Journal on Numerical Analysis}, 7\penalty0 (1):\penalty0
  1--46, 1970.

\bibitem[Gopalan and Blei(2013)]{gopalan2013efficient}
Prem~K Gopalan and David~M Blei.
\newblock Efficient discovery of overlapping communities in massive networks.
\newblock \emph{Proceedings of the National Academy of Sciences}, 110\penalty0
  (36):\penalty0 14534--14539, 2013.

\bibitem[Huang et~al.(2014)Huang, Sidiropoulos, and Swami]{huang2014non}
Kejun Huang, Nicholas Sidiropoulos, and Ananthram Swami.
\newblock Non-negative matrix factorization revisited: Uniqueness and algorithm
  for symmetric decomposition.
\newblock \emph{Signal Processing, IEEE Transactions on}, 62\penalty0
  (1):\penalty0 211--224, 2014.

\bibitem[Huang et~al.(2016)Huang, Fu, and Sidiropoulos]{huang2016anchor}
Kejun Huang, Xiao Fu, and Nikolaos~D Sidiropoulos.
\newblock Anchor-free correlated topic modeling: Identifiability and algorithm.
\newblock In \emph{Advances in Neural Information Processing Systems}, pages
  1786--1794, 2016.

\bibitem[Kaufmann et~al.(2016)Kaufmann, Bonald, and
  Lelarge]{kaufmann2016spectral}
Emilie Kaufmann, Thomas Bonald, and Marc Lelarge.
\newblock A spectral algorithm with additive clustering for the recovery of
  overlapping communities in networks.
\newblock In \emph{International Conference on Algorithmic Learning Theory},
  pages 355--370. Springer, 2016.

\bibitem[Kuang et~al.(2015)Kuang, Yun, and Park]{kuang2015symnmf}
Da~Kuang, Sangwoon Yun, and Haesun Park.
\newblock Symnmf: nonnegative low-rank approximation of a similarity matrix for
  graph clustering.
\newblock \emph{Journal of Global Optimization}, 62\penalty0 (3):\penalty0
  545--574, 2015.

\bibitem[Lei et~al.(2015)Lei, Rinaldo, et~al.]{lei2015consistency}
Jing Lei, Alessandro Rinaldo, et~al.
\newblock Consistency of spectral clustering in stochastic block models.
\newblock \emph{The Annals of Statistics}, 43\penalty0 (1):\penalty0 215--237,
  2015.

\bibitem[Ley(2002)]{ley2002dblp}
Michael Ley.
\newblock The dblp computer science bibliography: Evolution, research issues,
  perspectives.
\newblock In \emph{International symposium on string processing and information
  retrieval}, pages 1--10. Springer, 2002.

\bibitem[Lu et~al.(2015)Lu, Sun, Wen, Cao, and La~Porta]{lu2015algorithms}
Zongqing Lu, Xiao Sun, Yonggang Wen, Guohong Cao, and Thomas La~Porta.
\newblock Algorithms and applications for community detection in weighted
  networks.
\newblock \emph{Parallel and Distributed Systems, IEEE Transactions on},
  26\penalty0 (11):\penalty0 2916--2926, 2015.

\bibitem[Mcauley and Leskovec(2014)]{mcauley2014discovering}
Julian Mcauley and Jure Leskovec.
\newblock Discovering social circles in ego networks.
\newblock \emph{ACM Transactions on Knowledge Discovery from Data (TKDD)},
  8\penalty0 (1):\penalty0 4, 2014.

\bibitem[McSherry(2001)]{mcsherry2001spectral}
Frank McSherry.
\newblock Spectral partitioning of random graphs.
\newblock In \emph{Foundations of Computer Science, 2001. Proceedings. 42nd
  IEEE Symposium on}, pages 529--537. IEEE, 2001.

\bibitem[Minc(1988)]{minc1988nonnegative}
Henryk Minc.
\newblock Nonnegative matrices.
\newblock 1988.

\bibitem[Munkres(1957)]{munkres1957algorithms}
James Munkres.
\newblock Algorithms for the assignment and transportation problems.
\newblock \emph{Journal of the society for industrial and applied mathematics},
  5\penalty0 (1):\penalty0 32--38, 1957.

\bibitem[Press et~al.(1992)Press, Teukolsky, Vetterling, and
  Flannery]{press92numerical}
William~H. Press, Saul~A. Teukolsky, William~T. Vetterling, and Brian~P.
  Flannery.
\newblock \emph{Numerical Recipes in {C}}.
\newblock Cambridge University Press, 2nd edition, 1992.

\bibitem[Psorakis et~al.(2011)Psorakis, Roberts, Ebden, and Sheldon]{BNMF2011}
Ioannis Psorakis, Stephen Roberts, Mark Ebden, and Ben Sheldon.
\newblock Overlapping community detection using bayesian non-negative matrix
  factorization.
\newblock \emph{Phys. Rev. E}, 83:\penalty0 066114, Jun 2011.

\bibitem[Ray et~al.(2015)Ray, Ghaderi, Sanghavi, and
  Shakkottai]{ray2014overlap}
A.~Ray, J.~Ghaderi, S.~Sanghavi, and S.~Shakkottai.
\newblock Overlap graph clustering via successive removal.
\newblock In \emph{2014 52nd Annual Allerton Conference on Communication,
  Control, and Computing, Allerton 2014}, pages 278--285, United States, 1
  2015. Institute of Electrical and Electronics Engineers Inc.
\newblock \doi{10.1109/ALLERTON.2014.7028467}.

\bibitem[Sinha et~al.(2015)Sinha, Shen, Song, Ma, Eide, Hsu, and
  Wang]{sinha2015overview}
Arnab Sinha, Zhihong Shen, Yang Song, Hao Ma, Darrin Eide, Bo-june~Paul Hsu,
  and Kuansan Wang.
\newblock An overview of microsoft academic service (mas) and applications.
\newblock In \emph{Proceedings of the 24th international conference on world
  wide web}, pages 243--246. ACM, 2015.

\bibitem[Soundarajan and Hopcroft(2012)]{soundarajan2012using}
Sucheta Soundarajan and John Hopcroft.
\newblock Using community information to improve the precision of link
  prediction methods.
\newblock In \emph{Proceedings of the 21st international conference companion
  on World Wide Web}, pages 607--608. ACM, 2012.

\bibitem[Tang et~al.(2013)Tang, Sussman, Priebe, et~al.]{tang2013universally}
Minh Tang, Daniel~L Sussman, Carey~E Priebe, et~al.
\newblock Universally consistent vertex classification for latent positions
  graphs.
\newblock \emph{The Annals of Statistics}, 41\penalty0 (3):\penalty0
  1406--1430, 2013.

\bibitem[Wang et~al.(2011)Wang, Li, Wang, Zhu, and Ding]{wang2011community}
Fei Wang, Tao Li, Xin Wang, Shenghuo Zhu, and Chris Ding.
\newblock Community discovery using nonnegative matrix factorization.
\newblock \emph{Data Mining and Knowledge Discovery}, 22\penalty0 (3):\penalty0
  493--521, 2011.

\bibitem[Wang et~al.(2016)Wang, Cao, Jin, Cao, and He]{wang2016supervised}
Xiao Wang, Xiaochun Cao, Di~Jin, Yixin Cao, and Dongxiao He.
\newblock The (un) supervised nmf methods for discovering overlapping
  communities as well as hubs and outliers in networks.
\newblock \emph{Physica A: Statistical Mechanics and its Applications},
  446:\penalty0 22--34, 2016.

\bibitem[Yu et~al.(2015)Yu, Wang, Samworth, et~al.]{yu2015useful}
Yi~Yu, Tengyao Wang, Richard~J Samworth, et~al.
\newblock A useful variant of the davis--kahan theorem for statisticians.
\newblock \emph{Biometrika}, 102\penalty0 (2):\penalty0 315--323, 2015.

\bibitem[Zhang et~al.(2014)Zhang, Levina, and Zhu]{zhang2014detecting}
Yuan Zhang, Elizaveta Levina, and Ji~Zhu.
\newblock Detecting overlapping communities in networks using spectral methods.
\newblock \emph{arXiv preprint arXiv:1412.3432}, 2014.

\end{thebibliography}
	}
	
	\section*{Appendix}
	\appendix
	
	\section{Identifiability}
%

\begin{lem} \label{minc} (Lemma 1.1 of \cite{minc1988nonnegative})
	The inverse of a nonnegative matrix matrix $\bM$ is nonnegative {\sl if and only if } $\bM$ is a generalized permutation matrix.
\end{lem}

\begin{proof}[Proof of Theorem~\ref{mmsb_iden}]
	Suppose there are two parameter settings $(\bTheta^{(1)}$, $\bB^{(1)}$, $\rho^{(1)})$ and
	$(\bTheta^{(2)}$, $\bB^{(2)}$, $\rho^{(2)})$ that yield the same
	probability matrix:
	\bas{
		\bP=\rho^{(1)} \bTheta^{(1)} \bB^{(1)} {\bTheta^{(1)}}^T=\rho^{(2)} \bTheta^{(2)} \bB^{(2)} {\bTheta^{(2)}}^T.
	}
	Pick up pure node indices set $\I_1$ of
	$\bTheta^{(1)}$ such that $\bTheta_{\I_1}^{(1)}=\bI$, and denote $\bM
	= \bTheta_{\I_1}^{(2)}$.
	Similarly, pick up pure node indices set $\I_2$ of
	$\bTheta^{(2)}$ such that $\bTheta_{\I_2}^{(2)}=\bI$,
	and let $\bW = \bTheta_{\I_2}^{(1)}$.
	
	Then
	\bas{
		\rho^{(1)}\bB^{(1)} = \rho^{(2)}\bM \bB^{(2)} \bM^T \ \text{ and }\ 
		\rho^{(1)} \bW \bB^{(1)} \bW^T = \rho^{(2)} \bB^{(2)}.  
	}
	Denote $\bT=\bM \bW$, then
	\begin{equation} \label{trans_eq}
		\bB^{(1)} = \frac{1}{\rho^{(1)}} \bM \rho^{(1)} \bW \bB^{(1)} \bW^T \bM^T = \bT \bB^{(1)} \bT^T.
	\end{equation}
	Note that $\bM \cdot \bone = \bTheta_{\I_1}^{(2)} \cdot \bone =\bone$ and $\bW \cdot \bone= \bTheta_{\I_2}^{(1)} \cdot \bone =
	\bone$, so $\bT \cdot \bone = \bM \bW \cdot \bone = \bone$. We can consider $\bT$ as a transition matrix of a Markov
	chain, whose states are the nodes of the graph.
	Keep applying equation (\ref{trans_eq}) to its RHS, we get 
	\bas{
		\bB^{(1)} = \bT^k \bB^{(1)} {\bT^k}^T,
	}
	which implies $\bB^{(1)} = \bT_{\infty} \bB^{(1)} \bT_{\infty}^T$, where $\bT_{\infty}=\lim\limits_{k\rightarrow \infty} \bT^k$.
	
	Given that $\bB^{(1)}$ has full rank $K$, we must have $\bT_{\infty}$ has full rank. Now we prove that stationary point of the Markov chain, $\bT_{\infty}$, must be identity matrix.
	
	The nodes of a finite-size Markov chain can be split into a finite number of communication classes, and possibly some
	transient nodes.
	\begin{enumerate}
	\item If a communication class has at least two nodes and is aperiodic, then the rows corresponding to those nodes
	in $\bT_{\infty}$ are the stationary distribution for that class. Hence, $\bT_{\infty}$ has identical rows, so it
	cannot be full rank.
	\item The probability of a Markov chain ending in a transient node goes to zero as the number of iterations $k$
	grows, so the column of $\bT_{\infty}$ corresponding to any transient node is identically zero. Again, this means
	that $\bT_{\infty}$ cannot be full rank.
	\end{enumerate}
	Hence, the only configuration in which $\bT_{\infty}$ has full rank is when it contains $K$ communication classes,
	each with one node. This implies that $\bT_{\infty}=\bI$, and hence $\bT=\bI$.
	Note that if the communication classes are periodic, we can consider $\bT^t$ where $t$ is the product of the periods
	of all the classes; the matrix $\bT^t$ is now aperiodic for all the communication classes, and the above argument
	still applies to $\bT_{\infty}=\lim\limits_{k\rightarrow \infty} ({\bT^t})^k$.
		
	As $\bI=\bT=\bM\bW$, $\bM$ and $\bW$ have full rank, then $\bM^{-1}=\bW$, which is the case that a nonnegative matrix $\bM$ has nonnegative inverse $\bW$, using Lemma \ref{minc}, we know that $\bM$ is a generalized permutation matrix, and note that each row of $\bM$ sums to 1, the scale goes away and thus $\bM$ is a permutation matrix, which implies $\bW$ is also a permutation matrix. 
	As largest element of ${\bB^{(1)}}$ and ${\bB^{(2)}}$ are equals as 1, we should have $\rho^{(1)} = \rho^{(2)}$ and thus ${\bB^{(1)}}=\bM {\bB^{(2)}} \bM^T$.
	
	Also since we have
	\bas{
		\rho^{(1)} \bB^{(1)} {\bTheta^{(1)}}^T
		&=\rho^{(1)} \bTheta_{\I_1}^{(1)} \bB^{(1)} {\bTheta^{(1)}}^T
		=\rho^{(2)} \bTheta_{\I_1}^{(2)} \bB^{(2)} {\bTheta^{(2)}}^T
		=\rho^{(2)} \bM \bB^{(2)} {\bTheta^{(2)}}^T \\
		& = \rho^{(2)} \bM \bB^{(2)} \bM^T \bM {\bTheta^{(2)}}^T 
		= \rho^{(1)} \bB^{(1)} \bM {\bTheta^{(2)}}^T,
	}
	left multiply $\left( \rho^{(1)} \bB^{(1)}\right)^{-1}$ on both sides, we have ${\bTheta^{(1)}}={\bTheta^{(2)}}\bM^T$.
	
	Thus we have shown that MMSB is identifiable up to a permutation.
\end{proof}

	\section{Uniqueness of SNMF for MMSB networks}
	\begin{lem}[\citet{huang2014non}]
	\label{symNMFiff}
	If $\text{rank}(\bP)=K$, the Symmetric NMF $\bP=\bW\bW^T$ is unique if and only if the non-negative orthant is the only self-dual simplicial cone $\mathcal{A}$ with $K$ extreme rays that satisfies $\cone({\bW}^T) \subseteq \A=\A^*$, where $\mathcal{A}^*$ is the dual cone of $\mathcal{A}$, defined as $\mathcal{A}^*=\{ \by | \bx^T\by \geq 0, \forall \bx \in \mathcal{A} \}$.
\end{lem}

\begin{proof} [Proof of Theorem~\ref{unique}]
	When $\bB$ is diagonal, it has a square root  $\bC=\bB^{1/2}$, where $\bC$ is also a positive diagonal   matrix.   It is easy to see that $\cone(\bC)$ is the non-negative   orthant $\R_+^K$, so we have  
	\bas{
		\cone(\bW^T)=\cone(\bC^T \bTheta^T)=\cone(\bC^T) =\cone(\bC)=\R_+^K={\R_+^K}^*.
	}
	The second equality follows from the fact that $\bTheta$ contains all   pure nodes, and other nodes are convex combinations of these pure   nodes. The fourth equality is due to the diagonal form of $\bC$.   
	
	To see that this is unique, suppose there is another   self-dual simplicial cone satisfying $\cone({\bW}^T) \subseteq   \A=\A^*$. Then we have $\R_+^K \subseteq \A$ and $\A=\A^*\subseteq   {\left(\R_+^K\right)}^* = \R_+^K$, which implies $\A=\R_+^K$. 
	
	Hence, by Lemma~\ref{symNMFiff}, an identifiable MMSB model with a diagonal $\bB$ is sufficient for the Symmetric NMF solution to be unique and correct.
\end{proof}
	
	\section{Concentration of the Laplacian}
%
%

 We will use $X=c(1\pm \epsilon)$ to denote $X\in c[1-\epsilon,1+\epsilon]$ for ease of notation from now onwards.
 
\begin{lem} \label{lem:lln}
	For $\bTheta \in \R^{n \times K}$, where each row $\btheta_i \sim
	\dir(\balpha)$, $\forall j \in [K]$, 
	\bas{
		\sum_{i=1}^{n} \theta_{ij} =
		n\frac{\alpha_j}{\alpha_0}\bbb{1 \pm O_P\bbb{\sqrt{\frac{\alpha_0}{\alpha_j}\frac{\log
						n}{n}}}}
	}
	with probability larger than $1-1/n^3$.
	\begin{proof}
		By using Chernoff bound
		\bas{
			\uP\left(\left| \sum_{i=1}^{n} \theta_{ij}-n\frac{\alpha_j}{\alpha_0} \right|>\epsilon n\frac{\alpha_j}{\alpha_0}\right)\leq \exp\left(-\frac{\epsilon^2 n\frac{\alpha_j}{\alpha_0}}{3}\right),
		}
		so by setting $\epsilon=O_P\bbb{3\sqrt{\frac{{\log n}}{n{\alpha_j}/{\alpha_0}}}}$, 
		${
			\left| \sum_{i=1}^{n} \theta_{ij}-n\frac{\alpha_j}{\alpha_0} \right|\leq 3\sqrt{{\frac{\alpha_j}{\alpha_0}n\log n}},
		}$
		with probability larger than $1-1/n^3$.
		
		That is 
		\bas{
			\sum_{i=1}^{n} \theta_{ij} 
			&= n\frac{\alpha_j}{\alpha_0} \pm
			O_P\bbb{\sqrt{{\frac{\alpha_j}{\alpha_0}n\log n}}}
			=
			n\frac{\alpha_j}{\alpha_0}\bbb{1 \pm O_P\bbb{\sqrt{\frac{\alpha_0}{\alpha_j}\frac{\log
							n}{n}}}}.
		}
	\end{proof}
\end{lem}

\begin{lem} \label{lei} (Theorem 5.2 of \cite{lei2015consistency}) Let
	$\bA$ be the adjacency matrix of a random graph on $n$ nodes in which edges occur
	independently. Set $\uE[\bA] = \bP$ and assume that $n\max_{i,j} \bP_{ij} \leq d$ for
	$d \geq c_0 \log n$ and $c_0 > 0$. Then, for any $r > 0$ there exists a constant $C = C(r, c_0)$ such that:
	\bas{
		\uP(\left\|\bA-\bP\right\|\leq C\sqrt{d})\geq 1-n^{-r}.
	}
\end{lem}

\begin{fact} \label{mc}
	If $\bM$ is rank $k$, then $\|\bM\|_F^2\leq k \|\bM\|^2$.
\end{fact}

\begin{lem} \label{variant-davis}
	(Variant of Davis-Kahan \cite{yu2015useful}). Let $\bP$, $\hat{\bA} \in \R^{n\times n}$ be symmetric, with eigenvalues $\lambda_1 \geq \cdots \geq \lambda_n$ and $\hat{\lambda}_1 \geq \cdots \geq \hat{\lambda}_n$ respectively. Fix $1 \leq r \leq s \leq n$, and assume that $\min(\lambda_{r-1}-\lambda_{r},\lambda_{s}-\lambda_{s+1})>0$, where we define $\lambda_0 =  \infty$ and $\lambda_{n+1} =-\infty$.
	Let $d = s-r +1$, and let $\bV = (\bv_r, \bv_{r+1}, \cdots, \bv_s)\in \R^{n \times d}$ and $\hat{V} = (\hat{\bv}_r, \hat{\bv}_{r+1}, \cdots,\hat{\bv}_s)\in \R^{n \times d}$ have orthonormal columns satisfying $\bP \bv_j = \lambda_j \bv_j$ and $\hat{\bA} \hat{\bv}_j = \hat{\lambda}_j \hat{\bv}_j$ for $j = r,r + 1, \cdots,s$. Then
	there exists an orthogonal matrix $\hat{\bO} \in \R^{d \times d}$ such that
	\bas{
		\left\| \hat{\bV}-\bV\hat{\bO} \right\|_F \leq \frac{2^{3/2}\min\left(d^{1/2}\left\|\hat{\bA} - \bP\right\|, \left\|\hat{\bA}-\bP\right\|_F\right)}{\min(\lambda_{r-1}-\lambda_{r},\lambda_{s}-\lambda_{s+1})}
	}
\end{lem}

\begin{lem} \label{lem:sqrt_concentration}
	(Lemma A.1. of \citep{tang2013universally}). Let $\bH_1$, ${\bH_2} \in \R^{n\times n}$ be positive semidefinite with $\rank(\bH_1)=\rank(\bH_2)=K$. Let $\bX,\bY \in \R^{n\times K}$ be of full column rank such that $\bX\bX^T=\bH_1$ and $\bY\bY^T=\bH_2$. Let $\lambda_{K}\left(\bH_2\right))$ be the smallest nonzero eigenvalue of $\bH_2$. Then there exists an orthogonal matrix $\bR \in \R^{K\times K}$ such that:  
	\bas{
		\left\| {\bX\bR}-\bY \right\|_F \leq \frac{\sqrt{K}\left\|{\bH_1} - \bH_2\right\| \left(\sqrt{\left\|{\bH_1}\right\|}+\sqrt{\left\|\bH_2\right\|}\right)}{\lambda_{K}\left(\bH_2\right)}.
	}
\end{lem}

\begin{lem} \label{sp-bound}
	Recall that $\hat{\bA}_1=\hat{\bV}_1 \hat{\bE}_1 \hat{\bV}_1^T$ and $\bP_1=\bP(\cS,\cS)$  in Algorithm \ref{nmf-mmsb-pure}.
	If $\rho n = \Omega(\log n)$, then
	\bas{
		\left\|\hat{\bA}_1-\bP_1\right\|=O_P(\sqrt{\rho n}), \text{ and }
		\left\|\hat{\bA}_1-\bP_1\right\|_F=O_P(\sqrt{K\rho n})
	}
	with probability larger than $1-1/n^3$.
	\begin{proof}
		Lemma \ref{lei} gives the spectral bound of binary symmetric random matrices, in our model, 
		\bas{
			\frac{n}{2}\max_{i,j} \bP_1(i,j)\leq \frac{n}{2}\max_{i,j} \bP(i,j)=\frac{n}{2}\max_{i,j} \rho  \btheta_i \bB \btheta_j^T \leq \frac{n}{2}\max_{i,j} \rho  \btheta_i \bI \btheta_j^T\leq \rho \frac{n}{2}.
		}
		Note that we need to use $\bB$ is diagonal probability matrix and $\btheta_i$, $i\in[\frac{n}{2}]$ has $\ell_1$ norm 1 and all nonnegative elements for the last two inequality.
		
		Since $\rho n = \Omega(\log n)$, $\exists c_0 \geq 0$ that $\rho \frac{n}{2} \geq c_0 \log \frac{n}{2}$.
		
		Let $d=\rho \frac{n}{2}$, then $d \geq \frac{n}{2}\max_{i,j} \bP_1(i,j)$ and $d \geq c_0 \log \frac{n}{2}$, by Lemma \ref{lei}, $\forall r \geq 0$, $\exists$ $C>0$ that
		\bas{
			\uP\left(\left\|\bA_1-\bP_1\right\|\leq C\sqrt{\rho \frac{n}{2}} \right)\geq 1-(\frac{n}{2})^{-r},
		}
		where $\bA_1=\bA(\cS,\cS)$.
		So $\left\| \bA_1 - \bP_1 \right\|=O_P(\sqrt{\rho n})$, specially, taking $r=3$ then it is with probability larger than $1-1/n^3$.
		Hence 
		\bas{
			\left\|\hat{\bA}_1-\bP_1\right\| \leq \left\|\hat{\bA}_1-\bA_1 + \bA_1 -\bP_1\right\| \leq \left\|\hat{\bA}_1-\bA_1\right\| + \left\|\bA_1 -\bP_1\right\|=\hat{\sigma}_{K+1}+O_P(\sqrt{\rho n})=O_P(\sqrt{\rho n}),
		}
		where $\hat{\sigma}_{K+1}$ is the $(K+1)$-th eigenvalue of $\bA_1$ and is $O_P(\sqrt{\rho n})$ by Weyl's inequality.
		
		Since $\hat{\bA}_1$ and $\bP_1$ have rank $K$, then by Fact~\ref{mc}, 
		\bas{
			\left\|\hat{\bA}_1-\bP_1\right\|_F \leq \sqrt{2K}\left\|\hat{\bA}_1-\bP_1\right\|=O_P(\sqrt{K \rho n}).
		}
	\end{proof}
\end{lem}

\begin{lem} [Concentration of degrees]\label{cod}
Denote 
$\beta_{\mathrm{min}}=\min_a \bB_{aa}$.
Let $\bP = \rho \bTheta^{(1)} \bB {\bTheta^{(2)}}^T$, where 
$\rho$, $\bB$, $\bTheta^{(1)}\in \R^{\frac{n}{2}\times K}$, and $\bTheta^{(2)}\in \R^{\frac{n}{2}\times K}$
follow the restrictions of MMSB model. 
Let $\bD$ and $\D$ be diagonal matrices representing the sample and population node degrees. Then
\bas{
	\D_{ii}=O_P( \rho n/K),\ \D_{ii}=\Omega(\beta_{\mathrm{min}} \rho n/K)
	,\  \text{ and }\  \ \left| \bD_{ii}-\D_{ii} \right|=O_P(\sqrt{\rho n\log n / K})
}
with probability larger than $1-O_P(1/n^3)$.
	\begin{proof}
		$\forall i \in [\frac{n}{2}]$, we have 
		\bas{\D_{ii}
			&=\sum_{j=1}^{n/2} P_{ij} 
			= \sum_{j=1}^{n/2} \sum_{\ell=1}^{K} \rho \theta_{i\ell}^{(1)} \bB_{\ell\ell} {\theta}_{j\ell}^{(2)}
			\leq  
			\sum_{\ell=1}^{K}  \rho \theta_{i\ell}^{(1)}
			\sum_{j=1}^{n/2}{\theta}_{j\ell}^{(2)}
			= 
			\rho \sum_{\ell=1}^{K}\theta_{i\ell}^{(1)}\sum_{j=1}^{n/2}{\theta}_{j\ell}^{(2)}
			\tag{$\max_a \bB_{aa}=1$ by definition}\\
			&= \rho n\frac{\sum_{\ell=1}^{K}\theta_{i\ell}^{(1)}\alpha_l}{\alpha_0}
			\bbb{\frac{1}{2}+O_P\bbb{\sqrt{\frac{\alpha_0}{\alpha_l}\frac{\log
							n}{n}}}} \tag{from Lemma~\ref{lem:lln}} \\
			&=\frac{\rho n}{2K} \bbb{1+O_P\bbb{\sqrt{\frac{K\log n}{n}}}} 	\tag{when $\alpha_k=\frac{\alpha_0}{K}, \forall k\in[K]$},	
			}
		so $\D_{ii}=O_P(\rho n/K)$.
		
		Similarly,
		\bas{\D_{ii} 
			&\geq \sum_{\ell=1}^{K} \beta_{\mathrm{min}} \rho \theta_{i\ell}^{(1)}
			\sum_{j=1}^{n/2}{\theta}_{j\ell}^{(2)}
			= \beta_{\mathrm{min}}\rho \sum_{\ell=1}^{K}\theta_{i\ell}^{(1)}\sum_{j=1}^{n/2}{\theta}_{j\ell}^{(2)}\\
			&= \beta_{\mathrm{min}}\rho n\frac{\sum_{\ell=1}^{K}\theta_{i\ell}^{(1)}\alpha_l}{\alpha_0}\bbb{\frac{1}{2}+O_P\bbb{\sqrt{\frac{\alpha_0}{\alpha_l}\frac{\log
							n}{n}}}}  \tag{from Lemma~\ref{lem:lln}} \\
			&=\frac{\beta_{\mathrm{min}}}{2K}\rho n \bbb{1+O_P\bbb{\sqrt{\frac{K\log n}{n}}}} 	\tag{when $\alpha_l=\frac{\alpha_0}{K}, \forall l$},	
			}
		so $\D_{ii}=\Omega(\beta_{\mathrm{min}}\rho n/K)$.
		
		Then using Chernoff bound, we have 
		\bas{
			\uP\left(\left| \bD_{ii}-\D_{ii} \right|>\epsilon \D_{ii}]\right)\leq \exp\left(-\frac{\epsilon^2 \D_{ii}}{3}\right),
		}
		so when $\epsilon=O_P\bbb{3\sqrt{\frac{{K\log n}}{ \rho n}}}$, $\left| \bD_{ii}-\D_{ii} \right|\leq \epsilon \D_{ii}=O_P(\sqrt{\rho n\log n / K})$ with probability at least $1-1/n^3$. Note we have used Lemma~\ref{lem:lln}, so in total it is with probability larger than $1-O_P(1/n^3)$.

	\end{proof}
\end{lem}

\begin{lem} \label{lem:eigs}
	Denote 
	$\beta_{\mathrm{min}}=\min_a \bB_{aa}$.
	If $\rho n = \Omega(\log n)$, then 
	\bas{ \lambda_K(\bP_1)= \Omega(\beta_{\mathrm{min}}\rho n/K), \quad \lambda_1(\bP_1)= O_P\bbb{\rho n/K} }
	and
	\bas{
		\lambda_K(\bA_1)= \Omega(\beta_{\mathrm{min}}\rho n/K), \quad \lambda_1(\bA_1)= O_P\bbb{\rho n/K}
	}
	with probability larger than $1-O_P(K^2/n^3)$.
	\begin{proof}
		For conciseness, we will omit the subscript $1$ (see Lemma \ref{sp-bound}) in the following proof without loss of generality. 
		
		The $K$-th eigenvalue of $\bP$ is
		\bas{
			\lambda_K(\bP)&=\lambda_K(\rho \bTheta \bB \bTheta^T)=\lambda_K(\rho\bTheta \bB^{1/2} \bB^{1/2} \bTheta^T)=\lambda_K(\rho\bB^{1/2} \bTheta^T\bTheta \bB^{1/2}). 
		}
		Here we consider $\btheta_i$ as a random variable. 
		Denote 
		\bas{
			\hat{\bM}=\frac{1}{n/2}\rho\bB^{1/2} \bTheta^T\bTheta \bB^{1/2}=\frac{1}{n/2}\sum_{i=1}^{n/2} \rho\bB^{1/2} \btheta_i^T\btheta_i \bB^{1/2},
		} 
		then $\hat{\bM}_{ab} = \frac{1}{n/2}\sqrt{\beta_a \beta_b} \sum_{i=1}^{n/2}  \rho\theta_{ia}\theta_{ib}$. 
		
		Consider $\btheta_i \sim \dir(\balpha)$, 
		then
		\bas{
			\uE[\theta_{ia}\cdot \theta_{ib}]=
			\begin{cases} 
				\cov[\theta_{ia},\theta_{ib}]+\uE[\theta_{ia}]\cdot \uE[\theta_{ib}]=\frac{\alpha_a\alpha_b}{\alpha_0(\alpha_0+1)}, &\mbox{if } a \neq b \\ 
				\Var[\theta_{ia}]+\uE^2[\theta_{ia}]=\frac{\alpha_a(\alpha_a+1)}{\alpha_0(\alpha_0+1)}, & \mbox{if }  a = b 
			\end{cases}
		}
		so $\uE[\hat{\bM}_{ab}]=\sqrt{\beta_a \beta_b}\rho\uE[\theta_{ia}\cdot \theta_{ib}]\leq \rho$. And
		\bas{
			\uE[\hat{\bM}]=\rho(\mathrm{diag}(\bB\balpha) + \bB^{1/2}\balpha\balpha^T\bB^{1/2})/(\alpha_0(\alpha_0+1)).
		}
		
		Using Chernoff bound, we have
		\bas{
			\uP\left(\left| \hat{\bM}_{ab}-\uE[\hat{\bM}_{ab}] \right|>\epsilon \uE[\hat{\bM}_{ab}]\right)\leq \exp\left(-\frac{\epsilon^2 \frac{n}{2} \uE[\hat{\bM}_{ab}]}{3}\right)
			\leq \exp\left(-\frac{\epsilon^2 \rho {n} }{6}\right),
		}
		so when $\epsilon=O_P\bbb{\sqrt{\frac{18\log n}{\rho n}}}$, $\left| \hat{\bM}_{ab}-\uE[\hat{\bM}_{ab}] \right|\leq \epsilon \uE[\hat{\bM}_{ab}]$ with probability larger than $1-1/n^3$. Thus
		\bas{
			\left\| \hat{\bM} - \uE[\hat{\bM}] \right\| \leq \left\| \hat{\bM} - \uE[\hat{\bM}] \right\|_F \leq \sqrt{K^2 \epsilon^2 \uE^2[\hat{\bM}_{ab}]} \leq K \epsilon \rho.
		}
		Note that 
		\bas{\lambda_K({\uE[\hat{\bM}]})&=\rho\lambda_K{(\mathrm{diag}(\bB\balpha) + \bB^{1/2}\balpha\balpha^T\bB^{1/2})/(\alpha_0(\alpha_0+1))}\\
			&\geq \rho\left(\lambda_K(\mathrm{diag}(\bB\balpha)) + \lambda_K( \bB^{1/2}\balpha\balpha^T\bB^{1/2})\right)/(\alpha_0(\alpha_0+1))\\
			&= \rho\left(\min_{a}\beta_a\alpha_a+0\right)/(\alpha_0(\alpha_0+1))\\
			&= \rho\frac{\min_{a}\beta_a\alpha_a}{\alpha_0(\alpha_0+1)} \\
			&= \frac{\beta_{\mathrm{min}}\rho}{K(\alpha_0+1)}, \tag{when $\alpha_a=\frac{\alpha_0}{K}, \forall a$}	
		}
		the first inequality is by definition of the smallest eigenvalue and property of $\min$ function; the second equality is by the smallest eigenvalue of a $K\times K$ rank-1 matrix ($K>1$) is 0.

		By Weyl's inequality,
		\bas{
			\left| \lambda_K(\hat{\bM})-\lambda_K({\uE[\hat{\bM}]}) \right| \leq \left\| \hat{\bM} - \uE[\hat{\bM}] \right\|=O_P\left(K\sqrt{\frac{\rho\log n}{n}}\right),
		}
		so
		\bas{
			\lambda_K(\bP)&=\frac{n}{2}\lambda_K(\hat{\bM}) 
			\geq \frac{n}{2} \left(\frac{\beta_{\mathrm{min}}\rho}{K(\alpha_0+1)}- O_P\left(K\sqrt{\frac{\rho\log n}{n}}\right)\right)
		}
		with probability larger than $1-K^2/n^3$, 
		and thus $\lambda_K(\bP)= \Omega(\beta_{\mathrm{min}}\rho n/K)$.
		
		With similar argument we can get
		\bas{
			\lambda_1({\uE[\hat{\bM}]})&\leq \rho\left(1+\frac{\alpha_0}{K}\|\bbeta\|_1\right)/(K(\alpha_0+1)) \\
			&\leq \rho\left(1+{\alpha_0}\right)/(K(\alpha_0+1)) \\
			&= \frac{\rho}{K},
		}
		then $\lambda_1\bbb{\bP}= O_P\bbb{\rho n/K}$.
		
		From Weyl's inequality, we have
		\bas{
			\left| {\lambda_i({\bA})}-{\lambda_i(\bP)}  \right| \leq \left\|\bA-\bP\right\|=O_P\bbb{\sqrt{\rho n}},
		}
		so
		\bas{
			\lambda_K(\bA) \geq \lambda_K(\bP) - O_P\bbb{\sqrt{\rho n}} &\Longrightarrow \lambda_K(\bA) = \Omega(\beta_{\mathrm{min}}\rho n/K)\\
			\lambda_1(\bA) \leq \lambda_K(\bP) + O_P\bbb{\sqrt{\rho n}} &\Longrightarrow \lambda_1(\bA) = O_P\bbb{\rho n/K}.
		}
	\end{proof}
\end{lem}

\begin{lem} \label{ps-pertur}
	If $\rho n = \Omega(\log n)$, $\exists$ orthogonal matrix $\hat{\bO}_1 \in \mathbb{R}^{K\times K}$,
	\bas{
		\left\|\hat{\bV}_1-\bV_1\hat{\bO}_1\right\|_F=O_P\bbb{{\frac{K^{3/2}}{\beta_{\mathrm{min}}\sqrt{\rho n}}}}
	}
	with probability larger than $1-O_P(K^2/n^3)$.
	\begin{proof}
		From Lemma \ref{lem:eigs} we know that
		\bas{
			\lambda_K(\bP_1)= \Omega(\beta_{\mathrm{min}}\rho n/K)
		}
		with probability larger than $1-O_P(K^2/n^3)$.
		Because $\bP_1$ has rank $K$, its $K+1$ eigenvalue is 0, and the gap between the $K$-th and $(K+1)$-th eigenvalue of $\bP_1$ is $\delta= \Omega(\beta_{\mathrm{min}}\rho n/K)$. 
		Using variant of Davis-Kahan's theorem (Lemma \ref{variant-davis}), setting $r=1$, $s=K$, then $d=K$ is the interval corresponding to the first $K$ principle eigenvalues of $\bP_1$, we have $\exists\ \hat{\bO}_1 \in \mathbb{R}^{K\times K}$, 
		\bas{
			\left\|\hat{\bV}_1-\bV_1\hat{\bO}_1\right\|_F\leq \frac{2^{3/2}\min\left(\sqrt{K}\left\| \hat{\bA}_1-\bP_1 \right\|,\left\| \hat{\bA}_1-\bP_1 \right\|_F\right)}{\delta},
		}
		using Lemma \ref{sp-bound},
		\bas{
			\left\|\hat{\bV}_1-\bV_1\hat{\bO}_1\right\|_F 
			= O_P\bbb{\frac{2^{3/2}\sqrt{K\rho n}}{\beta_{\mathrm{min}} \rho n /K}}
			= O_P\bbb{{\frac{K^{3/2}}{\beta_{\mathrm{min}}\sqrt{\rho n}}}}
		}
		with probability larger than $1-O_P(K^2/n^3)$.
		
	\end{proof}
\end{lem}

\begin{lem} \label{E-lemma}
	If $\rho n = \Omega(\log n)$,
	then the orthogonal matrix $\hat{\bO}_1 \in \mathbb{R}^{K\times K}$
	of Lemma~\ref{ps-pertur} satisfies
	\bas{
		\left\|\hat{\bE}_1-\hat{\bO}_1^T\bE_1\hat{\bO}_1\right\|_F=O_P(\sqrt{K\rho n}/\beta_{\mathrm{min}})
	}
	with probability larger than $1-O_P(K^2/n^3)$.
	\begin{proof}
		We have
		\bas{
			\left\|\hat{\bV}_1\hat{\bE}_1\hat{\bV}_1^T-\bV_1\bE_1\bV_1^T\right\|_F = \left\|\hat{\bA}_1-\bP_1\right\|_F=O_P(\sqrt{K\rho n})
		}
		with probability larger than $1-1/n^3$ by Lemma \ref{sp-bound},
		and 
		\bas{
			\left\|\hat{\bV}_1-\bV_1\hat{\bO}_1\right\|_F=O_P\bbb{{\frac{K^{3/2}}{\beta_{\mathrm{min}}\sqrt{\rho n}}}}
		}
		with probability larger than $1-O_P(K^2/n^3)$ by Lemma \ref{ps-pertur}.
		Also,
		\bas{
			\hat{\bV}_1\hat{\bE}_1\hat{\bV}_1^T-\bV_1\bE_1\bV_1^T=\hat{\bV}_1 \left(\hat{\bE}_1-\hat{\bO}_1^T\bE_1\hat{\bO}_1\right) \hat{\bV}_1^T +\hat{\bV}_1\hat{\bO}_1^T\bE_1\left(\hat{\bO}_1\hat{\bV}_1^T-\bV_1^T\right) + \left(\hat{\bV}_1\hat{\bO}_1^T-\bV_1\right) \bE_1\bV_1^T.
		}
		So
		\bas{
			&\left\|\hat{\bE}_1-\hat{\bO}_1^T\bE_1\hat{\bO}_1\right\|_F=\left\|\hat{\bV}_1 \left(\hat{\bE}_1-\hat{\bO}_1^T\bE_1\hat{\bO}_1\right) \hat{\bV}_1^T\right\|_F \\
			\leq & \left\|\hat{\bV}_1\hat{\bE}_1\hat{\bV}_1^T-\bV_1\bE_1\bV_1^T\right\|_F+ \left\| \hat{\bV}_1\hat{\bO}_1^T\bE_1\left(\hat{\bO}_1\hat{\bV}_1^T-\bV_1^T\right) \right\|_F +  \left\| \left(\hat{\bV}_1\hat{\bO}_1^T-\bV_1\right) \bE_1\bV_1^T \right\|_F \\
			\leq & O_P(\sqrt{K\rho n}) + 2\left\| \bE_1\right\|\left\|\hat{\bV}_1-\bV_1\hat{\bO}_1^T \right\|_F \\
			=& O_P(\sqrt{K\rho n}) 
			+ O_P\bbb{\frac{ \rho n}{K}\cdot{\frac{K^{3/2}}{\beta_{\mathrm{min}}\sqrt{\rho n}}}} \\
			=&O_P\bbb{\sqrt{K\rho n}/\beta_{\mathrm{min}}},
			%
		}
		with probability larger than $1-O_P(K^2/n^3)$.
	\end{proof}
\end{lem}
\begin{lem} \label{lem:E_half}
	If  $\rho n = \Omega(\log n)$, 
	then $\exists$ an orthogonal matrix ${\bR}_1 \in \R^{K\times K}$, together with the orthogonal matrix $\hat{\bO} \in \mathbb{R}^{K\times K}$
	of Lemma~\ref{ps-pertur} satisfies
	\bas{
		\left\| \bR_1 \bE_1^{1/2}\hat{\bO}_1 - \hat{\bE}_1^{1/2} \right\|_F=O_P\bbb{K^{3/2}/\beta_{\mathrm{min}}^2}
	}with probability larger than $1-O_P(K^2/n^3)$.
	\begin{proof}
		From Lemma \ref{E-lemma} we have 
		\bas{\left\|\hat{\bO}_1\hat{\bE}_1\hat{\bO}_1^T-{\bE}_1\right\|_F
			=\left\|\hat{\bE}_1-\hat{\bO}_1^T\bE_1\hat{\bO}_1\right\|_F
			=O_P\bbb{{\sqrt{K\rho n}}/\beta_{\mathrm{min}}}.
		}
		with probability larger than $1-O_P(K^2/n^3)$.
		
		By Lemma \ref{lem:sqrt_concentration}, there exists an orthogonal matrix ${\bR}_1 \in \R^{d\times d}$  such that:  
		\bas{
			\left\| {\hat{\bO}_1\hat{\bE}_1^{1/2}{\bR_1}}- {\bE}_1^{1/2} \right\|_F 
			&\leq \frac{\sqrt{K}\left\|\hat{\bO}_1\hat{\bE}_1\hat{\bO}_1^T-{\bE}_1\right\| \left(\sqrt{\left\|\hat{\bO}_1\hat{\bE}_1\hat{\bO}_1^T\right\|}+\sqrt{\left\|{\bE}_1\right\|}\right)}{\lambda_{K}({\bE})}\\
			&\leq \frac{\sqrt{K}\cdot O_P\bbb{{{\sqrt{K\rho n}/\beta_{\mathrm{min}}}}}
				\left(O_P\bbb{\sqrt{\frac{\rho n}{K}}}+O_P\bbb{\sqrt{\frac{\rho n}{K}}}\right)}{\Omega(\beta_{\mathrm{min}}\rho n/K)} \tag{from Lemma~\ref{lem:eigs}}\\
			&\leq O_P\bbb{ K^{3/2}/\beta_{\mathrm{min}}^2}.
		}
		Note that 
		\bas{
			\left\| \bR_1 \bE_1^{1/2}\hat{\bO}_1 - \hat{\bE}_1^{1/2} \right\|_F
			= \left\|  \bE_1^{1/2} - \bR_1^T \hat{\bE}_1^{1/2}\hat{\bO}_1^T \right\|_F
			=\left\| {\hat{\bO}_1\hat{\bE}_1^{1/2}{\bR_1}}- {\bE}_1^{1/2} \right\|_F,
		}
		so
		\bas{
			\left\| \bR_1 \bE_1^{1/2}\hat{\bO}_1 - \hat{\bE}_1^{1/2} \right\|_F=O_P\bbb{ K^{3/2}/\beta_{\mathrm{min}}^2}.
		}
	\end{proof}
\end{lem}

\begin{proof}[Proof of Lemma~\ref{lem:maxnorm}]
		Note that if $\balpha=u\bone$, $u>0$, then
	\bas{
		\D_{21}(i,i)=
		\sum_{a\in[K],j\in \cS} \rho\theta_{ia}\bB_{aa}\theta_{ja}
		&=\sum_{a \in [K]} \rho\theta_{ia}\bB_{aa}\sum_{j\in \cS}\theta_{ja}
		=\frac{\rho n}{2K}\sum_{a\in[K]} \theta_{ia}\bB_{aa}\left(1\pm O_P\bbb{\sqrt{\frac{K\log n}{n}}}\right). \tag{by Lemma \ref{lem:lln}}
	}
	
	Because
	\bas{
		\left\| \sqrt{\rho} \cdot \be_i^T\D_{21}^{-1/2}\bTheta_2 \bB^{1/2} \right\|_F^2
		&=
		\frac{\rho \left\|\be_i^T\bTheta_2 \bB^{1/2} \right\|_F^2}{\D_{21}(i,i)} \\
		&\in \frac{\rho \sum_{a \in [K]} \theta_{ia}^2\bB_{aa}}{\frac{\rho n}{2K}\sum_{a \in [K]} \theta_{ia}\bB_{aa}} 
		\ccc{\frac{1}{1+O_P({\sqrt{{K\log n}/{n}}})},\frac{1}{1-O_P({\sqrt{{K\log n}/{n}}})}} \\
		&= \frac{2K}{ n} \cdot \frac{\sum_{a \in [K]} \theta_{ia}^2\bB_{aa}}{\sum_{a \in [K]} \theta_{ia}\bB_{aa}} \ccc{{1-O_P({\sqrt{{K\log n}/{n}}})},{1+O_P({\sqrt{{K\log n}/{n}}})}},
	}
	also note that
	\bas{
		\frac{2K}{ n}\cdot\frac{\sum_{a \in [K]} \theta_{ia}^2\bB_{aa}}{\sum_{a \in [K]} \theta_{ia}\bB_{aa}} 
		&\leq \frac{2K}{ n}\cdot {\max_a \theta_{ia}} 
		\leq 
		\frac{2K}{ n},
	}
	where the first inequality is an equality when  $\forall k \in [K]$, $\theta_{ik}=\max_a \theta_{ia}$ or $0$. The second inequality becomes an equality  when $\max_a \theta_{ia}=1$ (i.e. $i$ is a pure node). This implies that the LHS of the above equation equals $2K/n$ if and only if $i$ corresponds to a pure node.
		Then we have
	\bas{
		\left\| \sqrt{\rho} \cdot \be_i^T\D_{21}^{-1/2}\bTheta_2 \bB^{1/2} \right\|_F^2
		\leq \frac{2K}{ n}\cdot {\max_a \theta_{ia}}\bbb{1+O_P({\sqrt{{K\log n}/{n}}})},
	}
	and
	\bas{\left|
		\left\| \sqrt{\rho} \cdot \be_i^T\D_{21}^{-1/2}\bTheta_2 \bB^{1/2} \right\|_F^2
		- \frac{2K}{ n}\cdot\frac{ \sum_{a \in [K]} \theta_{ia}^2\bB_{aa}}{\sum_{a \in [K]} \theta_{ia}\bB_{aa} } \right| 
		&= \frac{2K}{ n}\cdot\frac{ \sum_{a \in [K]} \theta_{ia}^2\bB_{aa}}{\sum_{a \in [K]} \theta_{ia}\bB_{aa} } \cdot O_P({\sqrt{{K\log n}/{n}}}) \\
		&= O_P\bbb{\frac{2K}{n}\cdot \sqrt{{K\log n}/{n}}}
	}
	with probability larger than $1-O(1/n^3)$. 
	
	{ So $\left\| \sqrt{\rho} \cdot \be_i^T\D_{21}^{-1/2}\bTheta_2 \bB^{1/2} \right\|_F^2$ concentrates around $\frac{2K}{n}$ for pure nodes. Note that we implicitly assume that the impure nodes have $\max_a \theta_{ia}$ bounded away from one, and hence have norm bounded away from $2K/n$. }
	
	
	\end{proof}
\begin{proof}[Proof of Theorem~\ref{thm:entrywise}]
	Denote $\bTheta_{1}=\bTheta(\cS,:)$ and
	$\bTheta_{2}=\bTheta(\bar{\cS},:)$. 
Denote $\bA_{12}=\bA(\cS,\bar{\cS})$ and $\bA_{21}=\bA(\bar{\cS},{\cS})$, $\bD_{12}$ and $\bD_{21}$ are the (row) degree matrix of $\bA_{12}$ and $\bA_{21}$.
\OurAlgo projects $\bD_{21}^{-1/2}\bA_{21}$ onto $\hat{\bV}_1\hat{\bE}_1^{-1/2}$, and $\bD_{12}^{-1/2}\bA_{12}$ onto $\hat{\bV}_2\hat{\bE}_2^{-1/2}$. 

	Now, $\bV_1 \bE_1 \bV_1^T = \bP_1 = \rho\bTheta_1 \bB \bTheta_1^T$,
	with both $\bE_1$ and $\bB$ diagonal. This imples
	that there exists an orthogonal matrix $\bQ_1$  such that
	$\bV_1 \bE_1^{1/2} \bQ_1 = \sqrt{\rho} \cdot \bTheta_1 \bB^{1/2}$ (by Lemma A.1 of \cite{tang2013universally}).

	Also, as shown in Lemmas~\ref{ps-pertur}
	and~\ref{lem:E_half}, there exists orthogonal matrices $\hat{\bO}_1$ and and $\bR_1$ such that  
	\bas{
		\left\|\hat{\bV}_1-\bV_1\hat{\bO}_1\right\|_F=O_P\bbb{{\frac{K^{3/2}}{\beta_{\mathrm{min}}\sqrt{\rho n}}}},
		\quad \mbox{and } \quad
		\left\| \bR_1 \bE_1^{1/2}\hat{\bO}_1 - \hat{\bE}_1^{1/2} \right\|_F=O_P\bbb{ K^{3/2}/\beta_{\mathrm{min}}^2}
	}
	with probability larger than $1-O(K^2/n^3)$.

	Then we have:
	\bas{
		& \left\|\be_i^T\bP_{21}\bV_1\hat{\bO}_1 \hat{\bE}_1^{-1/2}-\sqrt{\rho} \cdot\be_i^T\bTheta_2 \bB^{1/2} \bQ_1^T \bR_1^T \right\|_F \\ =&\left\|\rho\cdot\be_i^T\bTheta_{2}\bB\bTheta_{1}^T\bV_1\hat{\bO}_1 \hat{\bE}_1^{-1/2}-\sqrt{\rho} \cdot\be_i^T\bTheta_2 \bB^{1/2} \bQ_1^T \bR_1^T \right\|_F 
		\tag{by $\bP_{21}=\rho\bTheta_{2}\bB\bTheta_{1}^T$} \\
		=& \left\|\rho\cdot\be_i^T\bTheta_{2}\bB^{1/2}\left(\bB^{1/2}\bTheta_{1}^T\right)\bV_1\hat{\bO}_1 \hat{\bE}_1^{-1/2}-\sqrt{\rho} \cdot\be_i^T\bTheta_2 \bB^{1/2} \bQ_1^T \bR_1^T \right\|_F \\
		=& \sqrt{\rho} \cdot \left\|\be_i^T\bTheta_{2}\bB^{1/2}\left(\bV_1\bE_1^{1/2}\bQ_1\right)^T\bV_1\hat{\bO}_1 \hat{\bE}_1^{-1/2}-\be_i^T\bTheta_2 \bB^{1/2} \bQ_1^T \bR_1^T \right\|_F \tag{by Lemma A.1 of \cite{tang2013universally}}\\
		=& \sqrt{\rho} \cdot \left\|\be_i^T\bTheta_{2}\bB^{1/2}\bQ_1^T\bR_1^T\left(\bR_1 \bE_1^{1/2}\hat{\bO}_1 \right) \hat{\bE}_1^{-1/2}-\be_i^T\bTheta_2 \bB^{1/2} \bQ_1^T \bR_1^T \right\|_F \\	
		=& \left\| \sqrt{\rho} \cdot \be_i^T\bTheta_{2}\bB^{1/2}\bQ_1^T\bR_1^T\left(\bR_1 \bE_1^{1/2}\hat{\bO}_1 - \hat{\bE}_1^{1/2} \right) \hat{\bE}_1^{-1/2} \right\|_F \\	
		\leq& \left\| \sqrt{\rho} \cdot\be_i^T\bTheta_{2}\bB^{1/2}\bQ_1^T\bR_1^T\right\|_F
		\left\| \bR_1 \bE_1^{1/2}\hat{\bO}_1 - \hat{\bE}_1^{1/2} \right\| \left\|\hat{\bE}_1^{-1/2} \right\| \\
		\leq&  \left\| \sqrt{\rho} \cdot\be_i^T\bTheta_{2}\bB^{1/2}\bQ_1^T\bR_1^T\right\|_F \cdot  O_P\bbb{ K^{3/2}/\beta_{\mathrm{min}}^2} \cdot O_P\bbb{\sqrt{\frac{K}{\beta_{\mathrm{min}}\rho n}}} \tag{by Lemmas~\ref{lem:maxnorm}, \ref{cod}, \ref{lem:eigs} and \ref{lem:E_half}} \\
	}
	\begin{equation}
		\Longrightarrow \left\|\be_i^T\bP_{21}\bV_1\hat{\bO}_1 \hat{\bE}_1^{-1/2}-\sqrt{\rho} \cdot\be_i^T\bTheta_2 \bB^{1/2} \bQ_1^T \bR_1^T \right\|_F
		=\left\| \sqrt{\rho} \cdot\be_i^T\bTheta_{2}\bB^{1/2}\bQ_1^T\bR_1^T\right\|_F \cdot O_P\bbb{{\frac{K^2}{  \beta_{\mathrm{min}}^{5/2}\sqrt{\rho n} }}}.  \label{equ:norm_tmp}
	\end{equation}
	
%
	
	Now that 
	\bas{
		& \left\|  \be_i^T\bD_{21}^{-1/2}\bA_{21}\hat{\bV}_1\hat{\bE}_1^{-1/2}-\sqrt{\rho} \cdot\be_i^T\D_{21}^{-1/2}\bTheta_2 \bB^{1/2} \bQ_1^T \bR_1^T \right\|_F \\
		\leq& \left\|  \be_i^T\bA_{21}\hat{\bV}_1\hat{\bE}_1^{-1/2}\left(1+O_P\bbb{\sqrt{K\log n/ n\rho}}\right)-\sqrt{\rho} \cdot\be_i^T\bTheta_2 \bB^{1/2} \bQ_1^T \bR_1^T \right\|_F / \sqrt{\D_{21}(i,i)} \tag{by Lemma \ref{cod}} \\
		\leq& 
		\text{ \footnotesize 
		$
		\frac{\bbb{1+O_P\bbb{\sqrt{\frac{K\log n}{n\rho}}}} \cdot \left\|  \be_i^T \left[ \left(\bA_{21}-\bP_{21}\right)\hat{\bV}_1 + \bP_{21}\left(\hat{\bV}_1-\bV_1\hat{\bO}_1\right) + \bP_{21}\bV_1\hat{\bO}_1   \right] \hat{\bE}_1^{-1/2} 
			- \sqrt{\rho} \cdot \be_i^T\bTheta_2 \bB^{1/2} \bQ_1^T \bR_1^T \right\|_F}
		{\sqrt{\D_{21}(i,i)}} 
		$
		}
		\\
		&+ O_P\bbb{\sqrt{K\log n/ n\rho}} \cdot \left\| \sqrt{\rho} \cdot\be_i^T\bTheta_2 \bB^{1/2} \bQ_1^T \bR_1^T \right\|_F/{\sqrt{\D_{21}(i,i)}} \\
		\leq &  \left(1+O_P\bbb{\sqrt{K\log n/ n\rho}}\right) \cdot  \left\{ \left\|  \be_i^T \left(\bA_{21}-\bP_{21}\right)\hat{\bV}_1\hat{\bE}_1^{-1/2}\right\|_F + \left\|\be_i^T\bP_{21}\left(\hat{\bV}_1-\bV_1\hat{\bO}_1\right)\hat{\bE}_1^{-1/2}\right\|_F \right. \\
		&\left. +\left\|\be_i^T\bP_{21}\bV_1\hat{\bO}_1 \hat{\bE}_1^{-1/2}-\sqrt{\rho} \cdot\be_i^T\bTheta_2 \bB^{1/2} \bQ_1^T {\bR}_1^T \right\|_F \right\} / {\sqrt{\D_{21}(i,i)}} +O_P\bbb{ K\sqrt{\log n / n^2 \rho}} \tag{by Lemmas~\ref{lem:maxnorm} and \ref{cod}} \\
		\leq& \bbb{1+O_P\bbb{\frac{1}{\beta_{\mathrm{min}}}\sqrt{K\log n/ n\rho}}}
		\cdot \left\{ O_P\bbb{\sqrt{K{\log n}}} \cdot  
		O_P\bbb{\sqrt{\frac{K}{\beta_{\mathrm{min}}\rho n}}}
		\right. \tag{by Azuma's and Lemma~\ref{lem:eigs}} \\
		&\left.+ O_P\bbb{\sqrt{ \frac{\rho n}{K}}} \cdot O_P\bbb{{\frac{K^{3/2}}{\beta_{\mathrm{min}}\sqrt{\rho n}}}} \cdot O_P\bbb{\sqrt{\frac{K}{\beta_{\mathrm{min}}\rho n}}}
		\right\}/{\sqrt{\beta_{\mathrm{min}}\rho n/K}} \tag{by Lemmas~\ref{cod}, \ref{ps-pertur}, and Eq.~\eqref{equ:norm_tmp}}  
		\\
		&+ \sqrt{\frac{2K}{n}} \cdot O_P\bbb{{\frac{K^2}{  \beta_{\mathrm{min}}^{5/2}\sqrt{\rho n} }}}
		+ O_P\bbb{ K\sqrt{\log n / n^2 \rho}} \\
		=&  O_P\bbb{\frac{K^{5/2}\sqrt{\log n}}{\beta_{\mathrm{min}}^{5/2}\rho n}}.
	}
	In the last step we use the fact that $ \left\| \be_i^T \left(\bA_{21}-\bP_{21}\right)\hat{\bV}_1\right\|_F^2$ is a sum of $K$ projections of $\be_i^T\left(\bA_{21}-\bP_{21}\right)$ on a fixed unit vector (since the eigenvectors come from the different partition of the graph). Now Azuma's inequality gives $ \left\| \be_i^T \left(\bA_{21}-\bP_{21}\right)\hat{\bV}_1\right\|_F=O_P(\sqrt{K\log n})$ with probability larger than $1-O(1/n^3)$. 
	
	Now as
	\bas{
		\left\|\sqrt{\rho} \cdot\be_i^T\D_{21}^{-1/2}\bTheta_2 \bB^{1/2}\right\|_F^2
		&=\frac{\rho\left\|{\be_i^T\bTheta_2 \bB^{1/2} }\right\|_F^2}{{\D_{21}(i,i)}}
		=\frac{\rho\cdot\be_i^T\bTheta_2 \bB \bTheta_2^T \be_i}{\D_{21}(i,i)}\\
		&= \frac{\bP_{2}(i,i)}{\D_{21}(i,i)}
		=\Omega\bbb{\frac{\rho}{\rho n/K}}
		=O_P\bbb{\frac{K}{n}},
	}
	let $\bm{O}=\bQ_1^T\bR_1^T$, then $\forall i$,
	\bas{
		\frac{
			\left\| \be_i^T\bD_{21}^{-1/2}\bA_{21}\hat{\bV}_1\hat{\bE}_1^{-1/2} -
			\sqrt{\rho} \cdot\be_i^T\D_{21}^{-1/2}\bTheta_2 \bB^{1/2} \bm{O} \right\|_F
		}{
			\left\|\sqrt{\rho}\cdot\be_i^T\D_{21}^{-1/2}\bTheta_2 \bB^{1/2}\right\|_F
		} 
		&=O_P\bbb{\frac{K^{5/2}\sqrt{\log n}}{\beta_{\mathrm{min}}^{5/2}\rho n} \cdot \sqrt{\frac{n}{K}}} \\
		&= O_P\left({\frac{{ K^2\sqrt{\log n}}}{\beta_{\mathrm{min}}^{5/2}\rho \sqrt{n}}}\right)
		}
	with probability larger than $1-n\cdot O(K^2/n^3)$=$1- O(K^2/n^2)$.
%
\end{proof}

	\section{Correctness of Pure node clusters}

\begin{proof}[Proof of Lemma~\ref{lem:errinPure}]
	Recall that $\max_i\|\bX_i\|$ concentrates around $\sqrt{2K/n}$ and this is achieved at the pure nodes.
	For ease of analysis let us introduce $\hy:=\sqrt{n/2K}\hx$ and $\bY=\sqrt{n/2K}\bX$.
	Recall that from Theorem~\ref{thm:entrywise} we have entry-wise consistency on 
	$\|\hy_i-\bY_i\bm{O}\| \leq \epsilon'= 
	\rowwise
	$ with probability larger than $1- O_P(K^2/n^2)$.
	
	Let $\epsilon_{\text{norm}}=O_P\bbb{\sqrt{\frac{K\log n}{n}}}=O_P(\epsilon')$ be the error of the norm of pure nodes in Lemma~\ref{lem:maxnorm}. 
	Then $\forall i\in \mathcal{F}$, 
	\bas{
		\|\hx_i\|\geq (1-\epsilon_0)\max_j \|\hx_j\|\geq (1-\epsilon_0)(1-\epsilon')\max_j \|\bX_j\|
		\geq (1-\epsilon_0)(1-\epsilon')(1-\epsilon_{\mathrm{norm}})\sqrt{2K/n}.
	}
	Hence we have a series of inequalities,
	\bas{
		(1-\epsilon_0)(1-\epsilon')(1-\epsilon_{\mathrm{norm}})  \leq \|\hy_i\| \leq \|\hy_i-\bY_i\bm{O}\|+\|\bY_i\|\leq \epsilon'+\|\bY_i\|.
	}
	
	Hence
	\bas{
		\|\bY_i\|^2 \geq (1-\epsilon_0-2\epsilon'-\epsilon_{\mathrm{norm}})^2\geq 1-2(\epsilon_0+2\epsilon'+\epsilon_{\mathrm{norm}})
	}
	
	And from the proof of Lemma \ref{lem:maxnorm},
	\bas{
		1-2(\epsilon_0+2\epsilon'+\epsilon_{\mathrm{norm}}) 
		&\leq \|\bY_i\|^2 
		\leq  \frac{\sum_{a\in[K]} \theta_{ia}^2\bB_{aa}}{\sum_{a\in[K]} \theta_{ia}\bB_{aa}} 
		\leq \max_a \theta_{ia} \bbb{1+\epsilon_{\mathrm{norm}}} \\
		\Longrightarrow \max_{a} \theta_{ia} &\geq 1-2(\epsilon_0+2\epsilon'+1.5\epsilon_{\mathrm{norm}})=1-O_P(\epsilon_0+\epsilon')
	}
	for $\epsilon=2(\epsilon_0+2\epsilon'+1.5\epsilon_{\mathrm{norm}})=O_P(\epsilon_0+\epsilon')$,   with probability larger than $1- O_P(K^2/n^2)$. 
	
	{Note that $\|\bX_i\|^2$ of those nearly pure nodes with $\max_a \theta_{ia}\geq 1-\epsilon$ 
	also concentrate around $\frac{2K}{n}$. These nearly pure nodes can also be used along with the pure nodes to recover the MMSB model asymptotically correctly.}
\end{proof}

\begin{lem} \label{lem:degree_sample_pure}
	Let $\mathcal{F}$ be the set of nodes with $\|\hx_i\|\geq (1-\epsilon_0)\max_j \|\hx_j\|$. Then when $\epsilon_0=O_P(\epsilon')$ and $\epsilon=O_P(\epsilon_0+\epsilon')$ from Lemma~\ref{lem:errinPure},
	\bas{
		\min_{i\in \mathcal{F}}\bD_2(i,i)= \frac{\rho n}{2K}\bbb{\beta_{\mathrm{min}}\pm O_P(\epsilon)},\quad
		\mbox{and} \quad
		\max_{i\in \mathcal{F}}\bD_2(i,i)= \frac{\rho n}{2K}\bbb{1\pm O_P(\epsilon)}
	} 
	with probability larger than $1-O_P(K^2/n^2)$.
	\begin{proof}
		From Lemma \ref{lem:errinPure} we know that $\forall i \in \mathcal{F}$, $\exists a_i$ that $\theta_{ia_i}\geq 1-\epsilon$, where $\epsilon=O_P(\epsilon_0+\epsilon')=O_P\bbb{\epsilon'}$.
		Then
		\bas{\D_2(i,i)
			&=\sum_{j=1}^{n/2} \bP_{ij} 
			= \sum_{j=1}^{n/2} \sum_{\ell=1}^{K} \rho \theta_{i\ell}^{(2)} \bB_{\ell\ell} {\theta}_{j\ell}^{(2)}
			= \rho \sum_{\ell=1}^{K}  \theta_{i\ell}^{(2)} \bB_{\ell\ell}\sum_{j=1}^{n/2}  {\theta}_{j\ell}^{(2)} \\
			&= \rho \sum_{\ell=1}^{K}  \theta_{i\ell}^{(2)} \bB_{\ell\ell} \frac{n}{2K}\bbb{1{ \pm}O_P\bbb{\sqrt{\frac{K\log n}{n}}}} \tag{from Lemma~\ref{lem:lln}}  \\
			&= \frac{\rho n}{2K}\bbb{(1-O_P\bbb{\epsilon})\bB_{a_i a_i}+O_P(\epsilon)}\bbb{1\pm O_P\bbb{\sqrt{\frac{K\log n}{n}}}} \\
			&=\frac{\rho n}{2K}\bbb{\bB_{a_i a_i}\pm O_P(\epsilon)}.
		}
		Using the proof in Lemma \ref{cod}, we have 
		${
			\bD_2(i,i)\in \D_2(i,i)\left[{1-O_P\bbb{\sqrt{\frac{K\log n}{\rho n}}}},{1+O_P\bbb{\sqrt{\frac{K\log n}{\rho n}}}}\right],
		}$
		so
		\bas{
			\bD_2(i,i)=\frac{\rho n}{2K}\bbb{\bB_{a_i a_i}\pm O_P(\epsilon)}.
		}
		Then
		\bas{
			\min_{i\in \mathcal{F}}\bD_2(i,i)&= \frac{\rho n}{2K}\bbb{\beta_{\mathrm{min}}\pm O_P(\epsilon)},\\
			\max_{i\in \mathcal{F}}\bD_2(i,i)&= \frac{\rho n}{2K}\bbb{1\pm O_P(\epsilon)}
		} 
		with probability larger than $1-O_P(K^2/n^2)$.
	\end{proof}
\end{lem}


\smallskip
\begin{proof}[Proof of Theorem~\ref{lem:clustering}]
	To prove this theorem, it is equivalent to prove that the upper bound of Euclidean distances within each community's (nearly) pure nodes is far more smaller than the lower bound of Euclidean distances between different communities' (nearly) pure nodes. 
	
	 Recall that  from Lemma~\ref{lem:errinPure}, for $i\neq j\in \mathcal{F}$, $\exists$ $a,b \in [K]$, such that $\theta_{ia}\geq 1-\epsilon$ and  $\theta_{jb}\geq 1-\epsilon$. 
	Note that $\epsilon=O_P(\epsilon_0+\epsilon')$ for $\epsilon'=\rowwise$ and $\epsilon_0=O_P(\epsilon')$.
	
	Using a similar argument as in the proof of Lemma \ref{lem:maxnorm}, we have:
	
		\bigskip \noindent
		1. if $a\neq b$, 
		\bas{
			\left\|\bY_i-\bY_j\right\|_2^2
			&\geq 
			\text{\footnotesize
				$
				\left[ \left( \frac{\theta_{ia}\bB_{aa}^{1/2}}{\sqrt{\sum_{k} \theta_{ik}\bB_{kk}}} - \frac{\theta_{ja}\bB_{aa}^{1/2}}{\sqrt{\sum_{k} \theta_{jk}\bB_{kk}}}\right)^2 
				+
				\left( \frac{\theta_{ib}\bB_{bb}^{1/2}}{\sqrt{\sum_{k} \theta_{ik}\bB_{kk}}} - \frac{\theta_{jb}\bB_{bb}^{1/2}}{\sqrt{\sum_{k} \theta_{jk}\bB_{kk}}}\right)^2
				\right] \cdot \bbb{1-O_P\bbb{\sqrt{\frac{K\log n}{n}}}}^2
				$
			}
			\\
			&\geq  
			\ccc{
				\bB_{aa}\left( \frac{1-\epsilon}{{\sqrt{\beta_{\mathrm{max}}}}} - \frac{\epsilon}{{\sqrt{\beta_{\mathrm{min}}}}}\right)^2
				+
				\bB_{bb}\left( \frac{1-\epsilon}{{\sqrt{\beta_{\mathrm{max}}}}} - \frac{\epsilon}{{\sqrt{\beta_{\mathrm{min}}}}}\right)^2
			} \cdot \bbb{1-O_P\bbb{\sqrt{\frac{K\log n}{n}}}}
			\\
			& \geq  
			\ccc{
				2\beta_{\mathrm{min}}\left( \frac{1-\epsilon}{{\sqrt{\beta_{\mathrm{max}}}}} - \frac{\epsilon}{{\sqrt{\beta_{\mathrm{min}}}}}\right)^2 
			} \cdot \bbb{1-O_P\bbb{\sqrt{\frac{K\log n}{n}}}}
			\\
			&\geq \ddd{2\frac{\beta_{\mathrm{min}}}{\beta_{\mathrm{max}}}\left[ 1 - \left( 1+ \sqrt{\frac{\beta_{\mathrm{max}}}{\beta_{\mathrm{min}}}}\right)\epsilon \right]^2 }
			\cdot \bbb{1-O_P\bbb{\sqrt{\frac{K\log n}{n}}}}
			\\
			&
			=
			\ddd{2\frac{\beta_{\mathrm{min}}}{\beta_{\mathrm{max}}}\left[ 1 - 2\left( 1+ \sqrt{\frac{\beta_{\mathrm{max}}}{\beta_{\mathrm{min}}}}\right)\epsilon + O_P(\epsilon^2) \right] }
			\cdot \bbb{1-O_P\bbb{\sqrt{\frac{K\log n}{n}}}}
			\\
			& = 
			{
				2\frac{\beta_{\mathrm{min}}}{\beta_{\mathrm{max}}}\bbb{1-O_P\bbb{\epsilon \sqrt{\frac{\beta_{\mathrm{max}}}{\beta_{\mathrm{min}}}}}}.
			}
		}
		So,
		\bas{
			\left\|\bX_i-\bX_j\right\|_2
			\geq \sqrt{2\frac{\beta_{\mathrm{min}}}{\beta_{\mathrm{max}}}}\sqrt{\frac{2K}{n}}\bbb{1-O_P\bbb{\epsilon \sqrt{\frac{\beta_{\mathrm{max}}}{\beta_{\mathrm{min}}}}}},
		}
		and then,
		\bas{
			\left\|\hx_i-\hx_j\right\|_2
			&\geq \left\|\bX_i-\bX_j\right\|_2 
				- \left\|\hx_i-\bX_i\right\|_2
				- \left\|\hx_j-\bX_j\right\|_2 \\
			&\geq 
			\sqrt{2\frac{\beta_{\mathrm{min}}}{\beta_{\mathrm{max}}}}\sqrt{\frac{2K}{n}}\bbb{1-O_P\bbb{\epsilon \sqrt{\frac{\beta_{\mathrm{max}}}{\beta_{\mathrm{min}}}}}}
			- 2\left\|\bX_i\right\|\cdot \epsilon' \\
			&\geq  \sqrt{2\frac{\beta_{\mathrm{min}}}{\beta_{\mathrm{max}}}}\sqrt{\frac{2K}{n}}\bbb{1-O_P\bbb{\epsilon \sqrt{\frac{\beta_{\mathrm{max}}}{\beta_{\mathrm{min}}}}}}
			- 2\sqrt{\frac{2K}{n}}\cdot \bbb{1+O_P(\epsilon)}\cdot \epsilon'\\
			&= 2 \sqrt{\frac{K\beta_{\mathrm{min}}}{n}}-O_P\bbb{\epsilon\sqrt{\frac{K}{n}}}. \tag{$\beta_{\mathrm{max}}=1$ by definition}
		}
		2. if $a = b$, first of all we have 
		\bas{
			(1-\epsilon)\beta_a \leq \sum_{k}\theta_{ik}\bB_{kk}\leq \beta_a + \epsilon\sum_{k\neq a} \beta_k,
		}
		then 
		\bas{
			\left\|\bY_i-\bY_j\right\|_2^2 
			&= \sum_{l}\left( \frac{\theta_{il}\bB_{ll}^{1/2}}{\sqrt{\sum_{k} \theta_{ik}\bB_{kk}}} - \frac{\theta_{jl}\bB_{ll}^{1/2}}{\sqrt{\sum_{k} \theta_{jk}\bB_{kk}}}\right)^2 
			\cdot \bbb{1+O_P\bbb{\sqrt{\frac{K\log n}{n}}}}^2
			\\
			& \leq \left[\bbb{\frac{\bB_{aa}^{1/2}}{\sqrt{(1-\epsilon)\beta_a}} - \frac{(1-\epsilon)\bB_{aa}^{1/2}}{\sqrt{\beta_a + \epsilon\sum_{k\neq a} \beta_k}}}^2 + \sum_{k\neq a} \frac{\beta_{\mathrm{max}}}{\beta_{\mathrm{min}}}\epsilon^2\right] 
			\cdot \bbb{1+O_P(\epsilon)}\\
			&=
			\text{\small 
			$
			\left\{
			\left[ 1+\frac{\epsilon}{2} + O_P(\epsilon^2) - (1-\epsilon)\left(1-\frac{\epsilon}{2}\frac{\sum_{k\neq a}\beta_k}{\beta_a}+O_P(\epsilon^2)\right)\right]^2 + (K-1)\frac{\beta_{\mathrm{max}}}{\beta_{\mathrm{min}}}\epsilon^2 
			\right\} 
			\cdot \bbb{1+O_P(\epsilon)} 
			$
			}
			\\
			& \leq
			\left\{
			 \left[\left(\frac{3}{2}+\frac{\sum_{k\neq a} \beta_{\mathrm{max}}}{2\beta_{\mathrm{min}}}\right)\epsilon + O_P(\epsilon^2)\right]^2 + O_P\bbb{K\frac{\beta_{\mathrm{max}}}{\beta_{\mathrm{min}}}\epsilon^2}
			\right\}
			\cdot \bbb{1+O_P(\epsilon)} \\
			& = O_P\bbb{K\frac{\beta_{\mathrm{max}}}{\beta_{\mathrm{min}}}\epsilon^2}.
		}
	So,
	\bas{
		\left\|\bX_i-\bX_j\right\|_2
		\leq \sqrt{\frac{2K}{n}}O_P\bbb{\epsilon\sqrt{K\frac{\beta_{\mathrm{max}}}{\beta_{\mathrm{min}}}}}
	}
	and then,
	\bas{
		\left\|\hx_i-\hx_j\right\|_2 
		&\leq \left\|\bX_i-\bX_j\right\|_2 
		+ \left\|\hx_i-\bX_i\right\|_2
		+ \left\|\hx_j-\bX_j\right\|_2 \\
		&\leq 
		\sqrt{\frac{2K}{n}}O_P\bbb{\epsilon\sqrt{K\frac{\beta_{\mathrm{max}}}{\beta_{\mathrm{min}}}}}
		+ 2\left\|\bX_i\right\|\cdot \epsilon' \\
		&\leq  \sqrt{\frac{2K}{n}}O_P\bbb{\epsilon\sqrt{K\frac{\beta_{\mathrm{max}}}{\beta_{\mathrm{min}}}}}
		+ 2\sqrt{\frac{2K}{n}}\cdot \bbb{1+O_P(\epsilon)}\cdot \epsilon'\\
		&=  O_P\bbb{\epsilon\sqrt{ \frac{K^2}{n\beta_{\mathrm{min}}}}} . \tag{$\beta_{\mathrm{max}}=1$ by definition}
	}

Now we can see $\frac{1}{2}\sqrt{\frac{K\beta_{\mathrm{min}}}{n}}$ can be used as a threshold to separate different clusters. However, in the algorithm we do not know ${\beta_{\mathrm{min}}}$ in advance, so we need to approximate it with some computable statistics. From Lemma \ref{lem:degree_sample_pure}, we know that $\bD_2(i,i)=\frac{\rho n}{2K}\bbb{\bB_{a_i a_i}\pm O_P(\epsilon')}$ when $\theta_{ia_i}\geq 1-\epsilon$.
So $\min_{i\in \mathcal{F}}\bD_2(i,i)$ and $\max_{i\in \mathcal{F}}\bD_2(i,i)$  can be used to estimate $\beta_{\mathrm{min}}$. 
\bas{
	\tau=\sqrt{\frac{K}{4n}\frac{\min_{i\in \mathcal{F}}\bD_2(i,i)}{\max_{i\in \mathcal{F}}\bD_2(i,i)}}
	= \sqrt{\frac{K}{4n}\frac{\rho n \bbb{\beta_{\mathrm{min}}\pm O_P(\epsilon)} /(2K)}{\rho n \bbb{1\pm O_P(\epsilon)} /(2K)} }
	= \frac{1}{2}\sqrt{\frac{K\beta_{\mathrm{min}}}{n}}\pm O_P\bbb{ \epsilon\sqrt{\frac{K}{n\beta_{\mathrm{min}}}}}.
}
Clearly,
\bas{
	2 \sqrt{\frac{K\beta_{\mathrm{min}}}{n}}\pm O_P\bbb{\epsilon\sqrt{\frac{K}{n}}}
	> 2\tau
	\gg  O_P\bbb{\epsilon\sqrt{ \frac{K^2}{n\beta_{\mathrm{min}}}}},
}
which means $\text{\clusteringAlgo}(\hx(\mathcal{F},:),\tau$)  can exactly give us $K$ clusters of different (nearly) pure nodes and return one (nearly) pure node from each of the $K$ clusters with probability larger than $1- O_P(K^2/n^2)$.
\end{proof}

	\section{Consistency of inferred parameters}
	


\begin{proof}[Proof of Theorem~\ref{thm:hxpinv}]
Let $\hy:=\sqrt{n/2K}\hx$ and $\bY=\sqrt{n/2K}\bX$.
Let $\epsilon=O_P(\epsilon_0+\epsilon')=O\bbb{\epsilon'}$ from Lemma~\ref{lem:errinPure}, 
where we show that $\|\bY_i\|^2\geq 1-\epsilon$ for $i\in \cS_P$. 
Furthermore for ease of exposition let us assume that the pure nodes are arranged so that $\hat{\bTheta}_{2p}=\hat{\bTheta}_{2}(\cS_p,:)$ is close to an identity matrix, i.e., the columns are arranged with a particular permutation.

Thus $\|\byp\|_F^2=\sum_i \|\byp(i,:)\|^2\geq K(1-\epsilon)$ and so $\|\byp\|_F\geq \sqrt{K(1-\epsilon)}$.

We have also shown that 
$\|\byp(i,:)\|^2\leq 1+\epsilon$, so $\|\byp\|_F\leq \sqrt{K}\bbb{1+\epsilon}$.

We will use
\ba{
	\label{eq:xpnorm}
	\|\hyp^{-1}-\bbb{\byp\bm{O}}^{-1}\|_F\leq \|\bbb{\byp\bm{O}}^{-1}(\byp\bm{O}-\hyp)\hyp^{-1}\|_F\leq  \|\byp^{-1}\|_F\|\byp\bm{O}-\hyp\|_F\|\hyp^{-1}\|.
}

First we will prove a bound on $\|\hyp^{-1}\|$. Let $\hat{\sigma}_i$ be the $i^{th}$ singular value of $\hyp$,
\ba{
	\label{eq:invnorm}
	\|\hyp^{-1}\|= \frac{1}{\hat{\sigma}_K}.
}
We can bound $\hat{\sigma}_K$ by bounding $\sigma_K$. In what follows we use $M_{1p}$ to denote the rows of $M_1$ indexed by $\cS_p$ when $M_1$ is $n/2\times K$ and by the square submatrix $M_1(\cS_p,\cS_p)$ is when $M_1$ is $n/2\times n/2$.
Note that $\|\bTheta_{2p}-I\|_F= 
\sqrt{K}\epsilon
$, $\|\bB^{1/2}\|_F=O_P(\sqrt{K})$, $\|\bTheta_{2p}\|_F=O_P(\sqrt{K})$ and $\|\D^{-1}_{21p}\|_F=O_P({K^{3/2}}/\beta_{\mathrm{min}}\rho n),$
\bas{
	\sigma_i^2&=\lambda_i(\byp\byp^T)=\frac{\rho n}{2K} \lambda_i(\D_{21p}^{-1/2}\bTheta_{2p} \bB\bTheta_{2p}^T\D_{21p}^{-1/2})\\
	&=\frac{\rho n}{2K} \lambda_i(\bB^{1/2}\bTheta_{2p}^T\D_{21p}^{-1} \bTheta_{2p}\bB^{1/2})\\
	&=\frac{\rho n}{2K} \lambda_i\left(\bB^{1/2}\left(\D_{21p}^{-1}+(\bTheta_{2p}-I)^T\D_{21p}^{-1}\bTheta_{2p}+\D_{21p}^{-1}(\bTheta_{2p}-I)\right)\bB^{1/2}\right).
}
Note that the matrix $\bB^{1/2}\D_{21p}^{-1}\bB^{1/2}$ is a diagonal matrix with the $(i,i)^{th}$ diagonal element being $\beta_i/\D_{21p}(i,i)$.

With similar arguments in the proof of Lemma \ref{lem:degree_sample_pure}, we can get 
\bas{
	\D_{21p}(i,i)=\frac{ n\rho }{2 K}\bbb{\beta_{i}\pm O_P(\epsilon)},
}
so 
\bas{
	\lambda_K(\bB^{1/2}\D_{21p}^{-1}\bB^{1/2}) =\frac{2K}{\rho n}\bbb{1\pm O_P(\epsilon/\beta_{\mathrm{min}})}.
}
  
By Weyl's inequality and note that operator norm is less than or equal to Frobenius norm, it immediately gives us:
\bas{
	\left|\sigma_i^2-\frac{\rho n}{2K} \lambda_i\left(\bB^{1/2}\D_{21p}^{-1}\bB^{1/2}\right)\right|
	\leq& \frac{\rho n}{2K}\cdot  \left\|\bB^{1/2}(\bTheta_{2p}-I)^T\D_{21p}^{-1}\bTheta_{2p}\bB^{1/2}+\bB^{1/2}\D_{21p}^{-1}(\bTheta_{2p}-I)\bB^{1/2}\right\|\\
	\leq& \frac{\rho n}{2K}\cdot  \left\|\bB^{1/2}\right\|_F\cdot\left\|\bTheta_{2p}-I\right\|_F\cdot\left\|\D_{21p}^{-1}\right\|_F\cdot 2\left\|\bTheta_{2p}\right\|_F\cdot\left\|\bB^{1/2}\right\|_F\\
	=& O_P\bbb{ \frac{\rho n}{2K}\cdot \sqrt{K} \cdot \sqrt{K}\epsilon \cdot
		 \frac{K^{3/2}}{\beta_{\mathrm{min}}\rho n} \cdot \sqrt{K}  \cdot \sqrt{K}  
		}\\
	=&
 	O_P\bbb{{\frac{K^{9/2}\sqrt{ \log n}}{\beta_{\mathrm{min}}^{7/2}\rho\sqrt{n}}}}
	\\
	\Longrightarrow \ \ \  \sigma_K^2=& {1\pm
	O_P\bbb{{\frac{K^{9/2}\sqrt{ \log n}}{\beta_{\mathrm{min}}^{7/2}\rho\sqrt{n}}}}
	}.
}

Now, Weyl's inequality for singular values gives us:
\bas{
	\left|\hat{\sigma}_i-\sigma_i\right|&\leq \|\hyp-\byp\bm{O}\|\leq \|\hyp-\byp\bm{O}\|_{F}
	=O_P(\sqrt{K}\epsilon)\\
	\hat{\sigma}_K&= \bbb{{1\pm
			O_P\bbb{{\frac{K^{9/2}\sqrt{ \log n}}{\beta_{\mathrm{min}}^{7/2}\rho\sqrt{n}}}}
	}}^{1/2}\left(1\pm
	O_P\left({\sqrt{K}\cdot\frac{{ K^2\sqrt{\log n}}}{\beta_{\mathrm{min}}^{5/2}\rho \sqrt{n}}}\right)
	\right)
	=1\pm O_P\bbb{{\frac{K^{9/2}\sqrt{ \log n}}{\beta_{\mathrm{min}}^{7/2}\rho\sqrt{n}}}}.
}
Plugging this into Equation~\eqref{eq:invnorm} we get:
\bas{
	\|\hyp^{-1}\| = 1 \pm O_P\bbb{{\frac{K^{9/2}\sqrt{ \log n}}{\beta_{\mathrm{min}}^{7/2}\rho\sqrt{n}}}}.
}
Finally putting everything together with Equation~\eqref{eq:xpnorm} we get:
\ba{
	\frac{\|\hxp^{-1}-\bbb{\bxp\bm{O}}^{-1}\|_F}{ \|\bxp^{-1}\|_F}
	=\frac{\|\hyp^{-1}-\bbb{\byp\bm{O}}^{-1}\|_F}{ \|\byp^{-1}\|_F}
	\leq \|\byp\bm{O}-\hyp\|_F\|\hyp^{-1}\|=
	\xpinverse
}
with probability larger than $1- O_P(K^2/n^2)$.
\end{proof}




\begin{proof}[Proof of Theorem~\ref{thm:theta}]
Recall that 
$\hat{\bTheta}_2 = \hat{\bTheta}(\bar{\cS}) = \bD_{12}^{1/2} \hx \hxp^{-1}
\bD_{21}^{-1/2}(\cSp,\cSp)$. First note that if one plugs in the population counterparts of the the terms in $\hat{\bTheta}_2$, then for some permutation matrix $\bpi$ that $\bTheta_{2p}:=\bTheta_2(\cS_p,:)\cdot\bpi$ is close to an identity matrix, and
\bas{
	\D_{21}^{1/2}\bX\bxp^{-1}\dpp^{-1/2}
	= \D_{21}^{1/2}
	\bbb{\sqrt{\rho}\cdot\D_{21}^{-1/2}\bTheta_2\bB^{1/2}}
	\bbb{\frac{1}{\sqrt{\rho}}\cdot\bB^{-1/2}\bpi\bTheta_{2p}^{-1}\dpp^{1/2}}
	\dpp^{-1/2}=\bTheta_2\bpi\bTheta_{2p}^{-1},
}
so
\bas{
	\bTheta_2\bpi=\D_{21}^{1/2}\bX\bxp^{-1}\dpp^{-1/2}\bTheta_{2p}.
}
We have the following decomposition
\bas{
	\left\|\hat{\bTheta}_2-\bTheta_2\bpi\right\|_F 
	\leq&
	\left\|(\bD_{21}^{1/2}-\D_{21}^{1/2})\hx \hxp^{-1} \dxp^{-1/2}\right\|_F
	+
	\left\|\D_{21}^{1/2}(\hat{\bX}-\bX\bm{O}) {\hxp}^{-1} \dxp^{-1/2}\right\|_F \\
	&+\left\|\D_{21}^{1/2}\bX\bm{O} ({\hxp}^{-1}-\bbb{\bxp\bm{O}}^{-1}) \dxp^{-1/2}\right\|_F 
	+\left\|\D_{21}^{1/2}\bX \bxp^{-1}( \dxp^{-1/2}-\dpp^{-1/2})\right\|_F \\
	&+\left\|\D_{21}^{1/2}\bX\bxp^{-1}\dpp^{-1/2}\bbb{\bI-\bTheta_{2p}}\right\|_F.
}
From the proof of Lemma~\ref{cod} we have 
$\sqrt{\bD_{21}(i,i)}
=\sqrt{\D_{21}(i,i)}(1\pm O_P(\sqrt{K\log n/n\rho}))$ and hence 
\bas{
	\eee{\bD_{21}^{1/2}-\D_{21}^{1/2}}=\eee{\D_{21}^{1/2}} O_P(\sqrt{K\log n/n\rho}),
}
and  
\bas{
	\eee{\dxp^{-1/2}-\dpp^{-1/2}}\leq \eee{\dpp^{-1/2}}O_P(\sqrt{K\log n/n\rho}).
}
And $\|\hxp^{-1}\|=\sqrt{n/(2K)}\|\hyp^{-1}\| = O_P(\sqrt{n/K})$, as we have shown in the proof of Theorem~\ref{thm:hxpinv}. Furthermore, by Fact~\ref{mc}, $\|\hxp^{-1}\|_F = O_P(\sqrt{n})$.

From the proof of Theorem \ref{thm:entrywise} we can get $\|\hat{\bX}-\bX\bm{O} \|_F=O_P(\sqrt{K}\epsilon')$.  
Theorem \ref{thm:hxpinv} gives
\bas{
	\|\hxp^{-1}-\bbb{\bxp\bm{O}}^{-1}\|_F
	&=\|\bxp^{-1}\|_F \cdot
	\xpinverse \\
	&= O_P\bbb{\sqrt{K}\cdot \sqrt{\frac{n}{K}}} \cdot \xpinverse 
	= O_P\bbb{{\frac{{ K^{5/2}\sqrt{\log n}}}{\beta_{\mathrm{min}}^{5/2} \rho}}}.
}

Also 
$\|\hx\|_F=O_P(\sqrt{K})$, since it concentrates around its population entry-wisely, and the max norm of any row of the population is $\sqrt{2K/n}$, so $\|\bX\|_F=O_P(\sqrt{K})$. 
And 
\bas{
	\|\D_{21}^{1/2}\bX\|
	&=\|\D_{21}^{1/2}\sqrt{\rho}\cdot \D_{21}^{-1/2}\bTheta_{2} \bB^{1/2} \|
	= \|\sqrt{\rho}\cdot\bTheta_{2} \bB^{1/2} \| 
	=  \sqrt{\|\bP\|}
	=  O_P(\sqrt{\rho n/K}).
}


 Hence, 
\bas{
	&\ \ \  \left\|(\bD_{21}^{1/2}-\D_{21}^{1/2}){\hx} {\hxp}^{-1} \dxp^{-1/2}\right\|_F\leq \left\|\bD_{21}^{1/2}-\D_{21}^{1/2}\right\| \left\|\hat{\bX}\right\|_F \left\|{\hxp}^{-1} \right\| \left\| \dxp^{-1/2}\right\|\\
	&= 
	\text{\small
		$
		O_P\bbb{\sqrt{\rho n/K} \cdot \sqrt{K\log n/n\rho}} \cdot O_P\bbb{\sqrt{K}} \cdot O_P\bbb{\sqrt{n/K}}\cdot O_P(\sqrt{K/\beta_{\mathrm{min}}\rho n}) 
		=O_P\bbb{\sqrt{K\log n/\beta_{\mathrm{min}}\rho}},
		$
	}
	\\
	&\ \ \ \left\|\D_{21}^{1/2}(\hat{\bX}-\bX\bm{O}) {\hxp}^{-1} \dxp^{-1/2} \right\|_F 
	\leq \left\|\D_{21}^{1/2}\right\|\left\|\hat{\bX}-\bX\bm{O}\right\|_F \left\|{\hxp}^{-1}\right\| \left\|\dxp^{-1/2} \right\|\\
	&=  
	\text{\small
		$
		O_P\bbb{\sqrt{\frac{\rho n}{K}}}\cdot 
		O_P\bbb{\sqrt{K}\epsilon'}
		\cdot O_P\bbb{\sqrt{\frac{n}{K}}} \cdot O_P\bbb{\sqrt{\frac{K}{\beta_{\mathrm{min}}\rho n}}}  =
		O_P\bbb{\sqrt{\frac{n}{\beta_{\mathrm{min}}}}\epsilon'}
		= O_P\bbb{{\frac{{ K^2\sqrt{\log n}}}{\beta_{\mathrm{min}}^{3} {\rho }}}}
		,
		$
	}
	\\
	&\ \ \   \left\|\D_{21}^{1/2}\bX ({\hxp}^{-1}-\bbb{\bxp\bm{O}}^{-1}) \dxp^{-1/2} \right\|_F 
	\leq \left\|\D_{21}^{1/2}\bX\right\| \left\|{\hxp}^{-1}-\bbb{\bxp\bm{O}}^{-1}\right\|_F \left\|\dxp^{-1/2} \right\|\\
	&= 
	\text{\small
		$
		O_P\bbb{\sqrt{\rho n/K}} \cdot 
		O_P\bbb{{\frac{{ K^{5/2}\sqrt{\log n}}}{\beta_{\mathrm{min}}^{5/2} {\rho }}}}
		\cdot O_P(\sqrt{K/\beta_{\mathrm{min}}\rho n}) =
		O_P\bbb{{\frac{{ K^{5/2}\sqrt{\log n}}}{\beta_{\mathrm{min}}^{3} {\rho }}}}
		, 
		$
	}
	\\
	&\ \ \   \left\|\D_{21}^{1/2}\bX \bxp^{-1}( \dxp^{-1/2}-\dpp^{-1/2}) \right\|_F \leq \left\|\D_{21}^{1/2}\bX\right\| \left\|\bxp^{-1}\right\| \left\| \dxp^{-1/2}-\dpp^{-1/2} \right\|_F \\
	&= 
	\text{\small
		$
		O_P\bbb{\sqrt{\rho n/K}} \cdot O_P(\sqrt{K}\cdot \sqrt{n/K})\cdot O_P(\sqrt{K/\beta_{\mathrm{min}}\rho n})  O_P\bbb{\sqrt{K\log n/ n\rho}} = O_P\bbb{\sqrt{K\log n/\beta_{\mathrm{min}}\rho}},
		$
	}
	\\
	&\ \ \   \left\|\D_{21}^{1/2}\bX\bxp^{-1}\dpp^{-1/2}\bbb{\bI-\bTheta_{2p}}\right\|_F
	\leq \left\|\D_{21}^{1/2}\bX\right\| \left\|\bxp^{-1}\right\| \left\|\dpp^{-1/2} \right\|_F \eee{\bI-\bTheta_{2p}}_F \\
	&= 
	\text{\small
		$
		O_P\bbb{\sqrt{\rho n/K}} \cdot O_P(\sqrt{K}\cdot \sqrt{n/K})\cdot O_P(\sqrt{K/\beta_{\mathrm{min}}\rho n}) \cdot \sqrt{K}\epsilon'
		=O_P\bbb{\sqrt{\frac{Kn}{\beta_{\mathrm{min}}}}\epsilon'}
		= O_P\bbb{{\frac{{ K^{5/2}\sqrt{\log n}}}{\beta_{\mathrm{min}}^{3} {\rho }}}}.
		$
	}
}

So
\bas{
	\left\|\hat{\bTheta}_2-\bTheta_2\bpi\right\|_F
	=O_P\bbb{{\frac{{ K^{5/2}\sqrt{\log n}}}{\beta_{\mathrm{min}}^{3} \rho}}}.
}
Since $\|\bTheta_2\|_F^2=\Omega (n/K)$, we finally have:
\bas{
	\frac{\left\|\hat{\bTheta}_2-\bTheta_2\bpi\right\|_F}{\|\bTheta_2\|_F}\leq 
	\thetaerror
}
with probability larger than $1- O(K^2/n^2)$.
\end{proof}

\begin{proof}[Proof of Theorem~\ref{thm:beta}]
Recall that	$\hat{\rho}\hat{\beta}_{a} =
	\left\|\be_a^T\bD_{21}^{1/2}(\cSp,\cSp)\hxp
	\right\|_F^2$, and  for some permutation matrix $\bpi$ that $\bTheta_{2p}:=\bTheta_2(\cS_p,:)\cdot\bpi$ is close to an identity matrix,
	if one plugs in the population counterparts of the the terms in $\hat{\bTheta}_2$, 
	\bas{
		\left\|\be_a^T\D_{21}^{1/2}(\cSp,\cSp)\bX_p
		\right\|_F
		&=\left\|\sqrt{\rho}\cdot\be_a^T\bTheta_{2p}\bpi^T\bB^{1/2}
		\right\|_F
		=\eee{\sqrt{\rho}\cdot\be_a^T\bbb{\bTheta_{2p}-\bI}\bpi^T\bB^{1/2}+\sqrt{\rho}\cdot\be_a^T\bpi^T\bB^{1/2}}_F\\
		&\leq \eee{\sqrt{\rho}\cdot\be_a^T\bbb{\bTheta_{2p}-\bI}\bpi^T\bB^{1/2}}_F+\eee{\sqrt{\rho}\cdot\be_a^T\bpi^T\bB^{1/2}}_F \\
		&=\sqrt{\rho} \epsilon' +
		\sqrt{\rho{\beta}_{a'}},
	}
where $a'\in[K]$ satisfies $\bpi_{a'a}=1$.

Using the bounds mentioned in the proof of Theorem~\ref{thm:theta}, we have:
\bas{
	&\ \ \ \left\|\be_i^T (\bD_{21}^{1/2}\hx-\D_{21}^{1/2}\bX\bm{O})\right\|
	=\left\|\be_i^T( \bD_{21}^{1/2}-\D_{21}^{1/2})\hx+\be_i^T\D_{21}^{1/2}(\hx-\bX\bm{O})\right\|\\
	&\leq \left \|\left(\sqrt{\bD_{21}(i,i)}-\sqrt{\D_{21}(i,i)}\right)\be_i^T\hx\right\|+\left\|\sqrt{\D_{21}(i,i)}\be_i^T(\hx-\bX\bm{O})\right\|\\
	&\leq \left(\sqrt{\bD_{21}(i,i)}-\sqrt{\D_{21}(i,i)}\right) \bb{\left\|\be_i^T\bbb{\hx-\bX\bm{O}}\right\|+\left\|\be_i^T\bX\right\|}
	+\sqrt{\D_{21}(i,i)}\left\|\be_i^T(\hx-\bX\bm{O})\right\|\\
	&= O_P\bbb{\sqrt{\frac{n\rho}{K}} \cdot
	\sqrt{\frac{K\log n}{n\rho}}}
	\ccc{\xabsoluterow+O_P\bbb{\sqrt{\frac{K}{n}}}}
	+O_P\bbb{\sqrt{\frac{n\rho}{K}}}\cdot \xabsoluterow\\
	&= \betaabsoluteerror.
}

As a result, 
\bas{
	\left| \sqrt{\hat{\rho}\hat{\beta}_a} -  \sqrt{\rho\beta_{a'}}\right| &\leq 
	\betaabsoluteerror+\sqrt{\rho} \epsilon'=\betaabsoluteerror,
}
and note that $\rho\beta_{a'}=\Omega(\rho)$, we have 
\bas{
	\hat{\rho}\hat{\beta}_a \in  \rho\beta_{a'} \left[ 1 - 
	\betaerror,1+\betaerror
	\right]	
}
with probability larger than $1- O(K^2/n^2)$.
\end{proof}

\end{document}